\documentclass[twoside]{article}

\usepackage[preprint]{aistats2026}
\usepackage{fix-cm}
\usepackage{tabularx, array, adjustbox}
\usepackage[utf8]{inputenc}
\usepackage[T1]{fontenc}
\usepackage{microtype}
\usepackage{float}
\usepackage{graphicx}
\usepackage{algorithm}
\usepackage{algorithmic}
\usepackage{tikz}
\usepackage{enumitem}
\usepackage{xcolor}
\usepackage{comment}
\usepackage{fancyhdr}
\usepackage{natbib}
\usepackage{placeins}
\usepackage[hidelinks]{hyperref}
\usepackage[nameinlink,capitalise]{cleveref}
\usepackage{subcaption}
\usepackage{mathtools}
\usepackage{booktabs}
\usepackage{array}
\usepackage{amsmath,amssymb,amsthm,mathtools}
\usepackage[framemethod=tikz]{mdframed}
\usepackage{pifont}  
\newcommand{\cmark}{\ding{51}}  % ✓
\newcommand{\xmark}{\ding{55}}  % ✗

\usepackage{amsmath}

\usepackage[most]{tcolorbox}
\tcbset{enhanced, sharp corners}
\newtcolorbox{keybox}{colback=black!2, colframe=black!12, boxrule=0.3pt,
  left=6pt, right=6pt, top=4pt, bottom=4pt}

\newenvironment{smalleq*}{%
  \begingroup
  \footnotesize
  \setlength{\abovedisplayskip}{4pt}%
  \setlength{\belowdisplayskip}{4pt}%
  \setlength{\abovedisplayshortskip}{3pt}%
  \setlength{\belowdisplayshortskip}{3pt}%
  \allowdisplaybreaks
  \begin{align*}
}{%
  \end{align*}
  \endgroup
}

\theoremstyle{definition}

\newtheorem{assumption}{Assumption}

\theoremstyle{plain}

\newtheorem{lemma}{Lemma}
\newtheorem{proposition}{Proposition}
\newtheorem{corollary}{Corollary}
\theoremstyle{remark}

\newcommand{\toesize}{\fontsize{7.5pt}{9.0pt}\selectfont}
\newcommand{\MACI}{\mathrm{MACI}}
\newcommand{\CRIT}{\mathrm{CRIT}}
\newcommand{\CL}{\mathrm{CL}}

\newcommand{\Ov}{\mathrm{O}}

\newcommand{\Q}{\mathit{Q}}            % target-aligned evidence quality
\newcommand{\Ihat}{\widehat{I}}        % normalized information gain
\newcommand{\JS}{D_{\mathrm{JS}}}      % Jensen–Shannon divergence

\newif\ifanon             % \anontrue for anonymous; comment out for camera-ready
\anontrue

\begin{document}

% If your paper is accepted and the title of your paper is very long,
% the style will print as headings an error message. Use the following
% command to supply a shorter title of your paper so that it can be
% used as headings.
%
%\runningtitle{I use this title instead because the last one was very long}

% If your paper is accepted and the number of authors is large, the
% style will print as headings an error message. Use the following
% command to supply a shorter version of the author names so that
% they can be used as headings (for example, use only the surnames)
%
%\runningauthor{Surname 1, Surname 2, Surname 3, ...., Surname n}

\twocolumn[

\aistatstitle{Multi-Agent Collaborative Intelligence: Dual-Dial Control for Reliable LLM Reasoning}

\aistatsauthor{Edward Y. Chang,  
Ethan Y. Chang}

\aistatsaddress{Stanford University, UIUC} ]

\begin{abstract}
Multi-agent debate often wastes compute by using a fixed adversarial stance, aggregating without deliberation, or stopping on heuristics. We introduce $\MACI$, an active controller with two independent dials that decouple information from behavior: an information dial that gates evidence by quality, and a behavior dial that schedules contentiousness from exploration to consolidation. A moderator tracks disagreement, overlap, evidence quality, and argument quality, and halts when gains plateau. We provide theory-lite guarantees for nonincreasing dispersion and provable termination, with a budget-feasible scheduler. Across clinical diagnosis and news-bias tasks, $\MACI$ improves accuracy and calibration while reducing tokens, and converts residual uncertainty into precision RAG plans that specify what to retrieve next. We use a cross-family LLM judge ($\CRIT$) as a conservative soft weight and stop signal, validated for order invariance and judge-swap stability; stability depends on using high-capability judges. $\MACI$ turns debate into a budget-aware, measurable, and provably terminating controller.
\end{abstract}

\section{Introduction}
\label{sec:intro}

Despite intense interest, multi-agent debate has seen limited progress in theory or practice. Unmoderated or ad hoc collaboration lacks provable termination, calibrated uncertainty, and cost control \citep{cemri2025multiagentllmsystemsfail}. Common approaches use a fixed adversarial stance~\citep{liang-etal-2024-encouraging}, aggregate without deliberation~\citep{Wang2023SelfConsistency}, or stop on heuristics, which wastes compute and locks in early errors.

%\vspace{-.05in}
\paragraph{The control problem.}
Reliable multi-agent reasoning requires \emph{active modulation}. Early rounds should explore alternatives, and later rounds should consolidate on well-supported conclusions. This demands two independent controls—what information enters (evidence gating) and how agents interact (behavioral stance)—together with principled stopping that detects convergence.

%\vspace{-.05in}
\paragraph{Beyond single-dial modulation.}
Prior work introduced contentiousness modulation~\citep{Chang2023CRIT,SocraSynthChangCSCI2023}, showing that behavior matters, but a single dial is insufficient. Systems need concurrent control over information quality and interaction intensity, and they must stop when additional debate yields diminishing returns. The $\CRIT$ framework~\citep{Chang2023CRIT} evaluates arguments but has not been coupled to systematic stopping.

%\vspace{-.06in}
\paragraph{$\MACI$: Dual dials with measured stopping.}
We present $\MACI$ (Multi-Agent Collaborative Intelligence), an active controller with two orthogonal dials. The \emph{information} dial ($\tau$) gates evidence by quality, and the \emph{behavior} dial ($\CL$) schedules contentiousness from exploration to consolidation. A moderator tracks four signals—disagreement ($\JS$), overlap ($\Ov$)\footnote{Overlap is not synonymous with agreement; it measures convergence of \emph{evidence}, whereas disagreement measures convergence of \emph{beliefs}. We require both.}, evidence quality ($\Q$), and argument quality ($\CRIT$)—and halts when gains plateau. Beyond prediction, $\MACI$ translates residual uncertainty into precision RAG: targeted acquisition plans that specify what to retrieve next.

%\vspace{-.06in}
\paragraph{Guarantees and scheduler.}
Under gated averaging with smoothing, dispersion is nonincreasing and the plateau test terminates in $O(1/\varepsilon)$ rounds, improving to $O(\log(1/\varepsilon))$ under a contraction condition (Lem.~\ref{lem:monotone}, Prop.~\ref{prop:termination}, Appx.~\ref{app:theory-plateau}). A budget-feasible UCB scheduler over $(\Delta\alpha,\Delta\gamma)$ achieves $\tilde{O}(\sqrt{KT})$ no-regret with zero expected budget violation (Prop.~\ref{prop:ucb-regret}, Appx.~\ref{app:learned-scheduler}).

\begin{table*}[t!]
\centering
\caption{Multi-agent orchestration methods: architectural comparison. %MACI introduces orthogonal controls and principled stopping absent from prior work.
}
\vspace{-0.1in}
\label{tab:related_comparison}
\toesize
\begin{tabular}{lccccl}
\toprule
Method & Info Gate & Behavior Ctrl & Stopping & Signals & Notes \\
\midrule
\multicolumn{6}{l}{\emph{Aggregation methods}} \\
Majority Vote~\cite{choi2025debatevote} & \xmark & \xmark & Fixed (N) & -- & Passive aggregation \\
Self-Consistency~\cite{Wang2023SelfConsistency} & \xmark & \xmark & Fixed (N) & -- & Fixed sampling \\
Mixture-of-Agents~\cite{wang2024rethinking} & \xmark & \xmark & Heuristic & -- & Layered ensemble \\
\midrule
\multicolumn{6}{l}{\emph{Fixed-stance debate}} \\
Liang et al.~\cite{liang-etal-2024-encouraging} & \xmark & Fixed & Heuristic & -- & Fixed adversarial roles \\
Diverse MAD~\cite{liu2025breaking} & \xmark & Varied & Heuristic & -- & Prompting style variation \\
Multi-LLM Debate~\cite{estornell2024multillmdebate} & \xmark & Fixed & Heuristic & -- & Ad hoc termination \\
\midrule
\multicolumn{6}{l}{\emph{Adaptive orchestration}} \\
\textbf{MACI (ours)} & \cmark & Scheduled & Plateau & 4 & Dual-dial control \\
\bottomrule
\end{tabular}
\vspace{-0.1in}
\end{table*} 

%\vspace{-.06in}
\paragraph{Empirical highlights.}
On clinical diagnosis with 1{,}500 cases spanning $100{+}$ diseases, $\MACI$ improves Acc@1 by \textbf{+3.9\,pp} over majority vote and \textbf{+3.7\,pp} over fixed-contentiousness debate, reduces calibration error (ECE \textbf{0.081} vs.\ 0.103), and uses \textbf{19\% fewer} generation tokens. A cross-domain study on news bias (619 articles) shows portability: the same controls narrow the partisan gap by \textbf{68\%} without domain tuning. $\MACI$ also translates residual uncertainty into \emph{precision RAG} plans that specify what to retrieve next to improve decision quality.

%\vspace{-.06in}
\paragraph{Contributions.}
We operationalize these ideas into five contributions:
\begin{enumerate}[leftmargin=1.2em,itemsep=-1pt,topsep=0pt]
\item \emph{Dual control.} Two-dial orchestration with information-theoretic stopping, where $D(t)$ is nonincreasing and plateau tests terminate within bounded rounds (Lem.~\ref{lem:monotone}, Prop.~\ref{prop:termination}).

\item \emph{Judge robustness.} $\CRIT$ is order invariant and stable under cross-family judge swaps (Appx.~\ref{app:crit-robustness}; $n{=}100$, 2--3\% winner flips; $\alpha{=}0.68$); we use it as a soft weight and part of a compound stop rule, not a hard oracle.

\item \emph{Targeted acquisition.} Systematic collaboration yields predictions together with precision RAG plans for underspecified problems.
\item \emph{Empirical validation.} Improvements in accuracy, calibration, and efficiency across two domains, with complete ablations and transcripts.
\end{enumerate}

%\vspace{-.06in}
\paragraph{Evaluator reliability.}
We rely on $\CRIT$, a cross-family judge that scores evidence-, logic-, and relevance-grounded argument quality. We validate order invariance and judge-swap stability ($n{=}100$; 2--3\% winner flips) and report tentatively acceptable agreement (Krippendorff's $\alpha{=}0.68$). We therefore use $\CRIT$ as a soft aggregation weight and as part of a compound stop rule, rather than as a hard oracle. Stability holds with high-capability judges; weaker or poorly aligned judges may require larger $K$, tighter thresholds, or periodic human audits (Appx.~\ref{app:crit-robustness}).
\section{Related Work}
\label{sec:related}

%%%\vspace{-.06in}
\paragraph{Multi-agent debate and orchestration.}
Debate-style prompting aims to surface errors through agent 
opposition~\citep{Irving2018Debate}, yet implementations often collapse to majority opinions or underperform strong single 
models~\citep{wang2024rethinking,smit2024foundation,estornell2024multillmdebate,choi2025debatevoteyieldsbetter}. 
For example, Liang et al.~\citep{liang-etal-2024-encouraging} maintain fixed adversarial intensity throughout, wasting tokens when contention becomes unproductive. Diverse Multi-Agent Debate~\citep{liu2025breaking} varies agent prompting styles but lacks behavioral adaptation within debates. Recent audits reveal that multi-agent systems frequently fail without systematic orchestration~\citep{cemri2025multiagentllmsystemsfail}, motivating principled control mechanisms. 

Table~\ref{tab:related_comparison} reveals a critical gap: existing methods provide either no control or single-axis control (fixed/varied debate intensity), but none combine independent information gating with scheduled behavioral transitions. 
\emph{Aggregation methods} sample diverse responses but aggregate without deliberation, causing performance degradation as answer spaces 
grow~\citep{choi2025debatevoteyieldsbetter}. \emph{Fixed-stance debate} 
methods maintain constant intensity, unable to transition from exploratory disagreement to consolidative synthesis. 
$\MACI$ addresses both limitations: an information dial filters evidence by quality while a behavior dial schedules contentiousness from high to low, with information-theoretic stopping that detects convergence plateaus rather than relying on heuristic round limits. \textbf{To our knowledge, no prior work provides independent controls over information quality and behavioral stance.}

%%%\vspace{-.06in}
\paragraph{Computational paradigms and capability trade-offs.}
Transformers~\citep{vaswani2017attention} required orders of magnitude more computation than RNNs but unlocked scaling laws~\citep{kaplan2020scaling}; large language models \citep{brown2020language} continue this trend. $\MACI$ follows this paradigm: systematic orchestration requires overhead, but provides measurable gains in accuracy, calibration, and epistemic humility. For high-stakes applications, reliability justifies cost.

%%%\vspace{-.05in}
\paragraph{Ensemble methods and aggregation.}
Self-Consistency~\citep{Wang2023SelfConsistency} and majority voting aggregate sampled chains with fixed behavior. Mixture-of-Agents~\citep{wang2024moa,li2025rethinkingmixtureofLLMs} improves performance via layered ensembles but under-specifies convergence beyond heuristic rounds. These methods assume weak error correlation and rarely specify stopping criteria. On short-horizon multiple-choice tasks, voting can dominate when answer spaces are small~\citep{choi2025debatevoteyieldsbetter}, but performance degrades as dimensionality grows (Appendix~\ref{app:majority}). MACI differs by using information-theoretic signals (disagreement, overlap, evidence quality) to drive active deliberation with plateau-based stopping.

%%%\vspace{-.06in}
\paragraph{Control and stopping in AI systems.}
Control-theoretic approaches in reinforcement learning use measured feedback to adjust policy parameters~\citep{recht2019tour}. Active learning frameworks gate information acquisition based on uncertainty~\citep{settles2009active}. Our work extends these ideas to multi-agent reasoning: the moderator tracks four signals and adjusts behavioral stance when progress plateaus, analogous to adaptive control in dynamical systems. Prior debate work lacks this closed-loop structure, relying instead on fixed interaction patterns or ad hoc stopping rules.

%%%\vspace{-.05in}
\paragraph{Initialization and few-shot learning.}
Bayesian \citep{xie2022explanation}, meta-learning~\citep{dai2023gpt}, and optimization-as-inference~\citep{vonoswald2023transformers} perspectives explain in-context learning. Selection and weighting strategies adapt which demonstrations matter~\citep{hong2025moicl,wang2024bycs}. We use lightweight information-theoretic metrics (pattern density, prior-target mismatch) to initialize debate parameters. Appendix~\ref{app:bridge} validates these heuristics on controlled tasks; broader ICL theory appears in concurrent work.

%%%\vspace{-.05in}
\paragraph{Evaluation with LLM judges.}
Independent LLM judges are appealing, but evaluators can exhibit self-preference bias, especially when sharing a family with the debaters~\citep{panickssery2024llm}. Moreover, recent scalable-oversight studies report that weak-judge regimes yield inconsistent gains~\citep{kenton2024scalable}. \emph{Compounding this,} LLMs lack robust self-validation for constraint adherence and reasoning consistency~\citep{bommasani2022foundationmodels,PathAGIChang2024}, a limitation loosely analogous to Gödel-style barriers~\citep{godel1931english} and supported empirically~\citep{hong2024verificationabilities}. Hence we use cross-family judges, mask agent identity, require span-grounded justifications, and aggregate small judge panels ($K{\ge}3$) with reliability weighting; judge-swap and order-swap checks, plus a CRIT-free controller variant, indicate that stopping and weights are not driven by a single evaluator.

%Most LLMs are overly confident about their output \citep{prasad2025llmsdebateconf}. Godel limitation provides a good analogy that slef-evaluation runs into blind spots. Therefore, seeking for indipendent LLM jusges is a plausible avenue. Systems increasingly rely on LLM evaluators, raising concerns about self-preference bias~\citep{panickssery2024llm}. Scalable oversight studies find weak judge regimes produce inconsistent benefits~\citep{kenton2024scalable}. We mitigate these risks through cross-family judges (evaluator from different LLMs than debaters), span-grounded justifications, and reliability weighting that adapts to argument quality over time.
\section{MACI: Information and Behavior Dials with Measurable Stopping}
\label{sec:framework}

% helper: small math for blocks (place once)
\newenvironment{eqhalf}{
  \begingroup\small
  \setlength{\abovedisplayskip}{5.8pt}%
  \setlength{\belowdisplayskip}{1.5pt}%
  \setlength{\abovedisplayshortskip}{1.5pt}%
  \setlength{\belowdisplayshortskip}{1.5pt}%
}{\endgroup}

%\vspace{-.05in}
\paragraph{Goal.}
Coordinate multiple LLM agents to explore meaningfully and then consolidate on well-supported conclusions with calibrated confidence under a fixed token budget. All notation is defined in Appx.~\ref{app:signals}.

%\vspace{-.05in}
\paragraph{Failures we address.}
\begin{itemize}[leftmargin=1.2em,itemsep=-0.2pt,topsep=-0.2pt]
\item \textit{Prior capture:} Single models fixate on high-frequency patterns and confirm early impressions.
\item \textit{Majority illusion:} Voting appears strong in multiple choice, but does not test the depth of reasoning or the quality of evidence.
\item \textit{Unregulated debate:} Contentiousness drifts, evidence quality is not gated, and stopping is ad hoc.
\item \textit{Cost sprawl:} Long chains consume budget without proportional gains.
\end{itemize}

%\vspace{-.05in}
\paragraph{Design summary.}
$\MACI$ exposes two independent controls. The \textit{information dial} $(\tau_{\Q})$ gates which evidence agents may cite by enforcing a quality threshold. The \textit{behavior dial} ($\CL$) sets contentiousness, instructing agents how aggressively to challenge opponents. 
A moderator tracks four signals—target-aligned evidence quality ($\Q$), inter-agent disagreement ($\JS$), support overlap ($\Ov$), and argument quality ($\CRIT$)—and stops when they plateau jointly. The construction of the target prototype $\boldsymbol{\theta}$ used by $\Q$ is detailed in Appx.~\ref{app:theta}.

%\vspace{-.05in}
\subsection{The Control Loop}
\label{sec:framework:loop}

Algorithm~\ref{alg:maci} (in \text{Appx.~\ref{app:maci_algorithm}}) shows the three-phase cycle. 
Agents receive a composite anchor $A^{(t)}=\big(\psi^{(t)},\,CL^{(t)},\,\tau_{\Q}^{(t)},\,\tau_{\CRIT}^{(t)}\big)$
that contains the history of curated debates, current contentiousness instruction, and evidence quality thresholds. Each agent produces claim-justification pairs with probability distribution $p_u^{(t)}$ over the outcomes and evidence cited spans $\hat{\mathcal{E}}_u^{(t)}$. 
The moderator scores arguments with \textsc{CRIT} (Appendix~\ref{app:crit}), computes evidence quality $Q$ for the cited spans, and admits only those satisfying both gates, $Q \ge \tau_{\text{Q}}(t)$ and $\mathrm{CRIT} \ge \tau_{\text{CRIT}}(t)$ (defaults in Appx.~B.1); it then computes all signals and updates both dials.

%The moderator scores arguments with $\CRIT$ (Appendix~\ref{app:crit}), admits those exceeding $\tau^{(t)}$, computes all signals, and updates both dials. When information gain and disagreement reduction plateau, debate stops.

%\vspace{-.05in}
\subsection{Four Measurement Signals}
\label{sec:framework:signals}

%\vspace{-.05in}
\paragraph{Notation.}
Let $\mathcal{Y}$ be finite outcome space, $\Delta(\mathcal{Y})$ the probability simplex. Two agents $u \in \{A,B\}$ alternate over rounds $t=1,2,\ldots$. Each outputs $p_u^{(t)} \in \Delta(\mathcal{Y})$ and cites evidence spans $\hat{\mathcal{E}}_u^{(t)}$.

%\vspace{-.05in}
\paragraph{Signal 1: Evidence quality ($\Q$).}
Measures alignment between cited evidence and target distribution using encoder embeddings:
\begin{eqhalf}
\begin{equation}
\Q^{(t)} = \cos(\bar{\mathbf{e}}^{(t)}, \boldsymbol{\theta}),
\quad \bar{\mathbf{e}}^{(t)} = \frac{1}{|\mathcal{E}^{(t)}|} \sum_{e \in \mathcal{E}^{(t)}} \frac{\phi(e)}{\|\phi(e)\|},
\label{eq:q}
\end{equation}
\end{eqhalf}

where $\mathcal{E}^{(t)}$ collects the spans from \emph{admitted} arguments in round $t$, and $\boldsymbol{\theta}$ is the target prototype.
Rising $\Q^{(t)}$ indicates convergence on task-relevant evidence.

\textit{Admission rule.} For each \emph{candidate} argument $i$ at round $t$, we compute a per-argument evidence score $q_i^{(t)}$ (analogous to \eqref{eq:q}, using only the spans cited by $i$). The moderator admits $i$ iff it passes the \textbf{dual gate}:
\[
q_i^{(t)} \,\ge\, \tau_{\Q}(t)
\quad\textbf{and}\quad
\CRIT_i^{(t)} \,\ge\, \tau_{\CRIT}(t).
\]
After admission, $\mathcal{E}^{(t)}$ aggregates all admitted spans and $\Q^{(t)}$ summarizes the evidence quality of that batch. Default settings for $\tau_{\Q}$ and $\tau_{\CRIT}$ appear in Appx.~B.1.

%\vspace{-.05in}
\paragraph{Signal 2: Disagreement ($\JS$).}
Jensen-Shannon divergence quantifies distributional disagreement:
\begin{eqhalf}
\[
\JS^{(t)} = \tfrac{1}{2} \text{KL}(p_A^{(t)} \| m^{(t)}) + \tfrac{1}{2} \text{KL}(p_B^{(t)} \| m^{(t)}),
\]
\begin{equation}
m^{(t)} = \tfrac{1}{2}(p_A^{(t)} + p_B^{(t)}).
\label{eq:jsd}
\end{equation}
\end{eqhalf}

$\JS$ is symmetric and bounded in $[0,1]$ when using log base $2$, and remains finite under disjoint support. Falling $\JS$ tracks convergence; plateaus indicate exhausted productive disagreement (thresholds are selected on the development set; Appendix~\ref{app:signals}).

%\vspace{-.05in}
\paragraph{Signal 3: Support overlap ($\Ov$).}
Jaccard similarity of cited evidence sets:
\begin{eqhalf}
\begin{equation}
\Ov^{(t)} = \frac{|\hat{\mathcal{E}}_A^{(t)} \cap \hat{\mathcal{E}}_B^{(t)}|}{|\hat{\mathcal{E}}_A^{(t)} \cup \hat{\mathcal{E}}_B^{(t)}|}.
\label{eq:overlap}
\end{equation}
\end{eqhalf}

Rising overlap indicates agents grounding arguments in shared evidence rather than talking past each other. Combined with falling $\JS$, this signals genuine convergence rather than forced agreement.

%\vspace{-.05in}
\paragraph{Signal 4: Argument quality ($\CRIT$).}
Cross-family LLM evaluator scores each argument on [0,1] for logical coherence, evidence support, and reasoning depth (Appendix~\ref{app:crit}). The moderator maintains exponential moving average of per-agent quality:
\begin{eqhalf}
\begin{equation}
\begin{aligned}
\Gamma_u^{(t)} &= \lambda \Gamma_u^{(t-1)} + (1-\lambda) \overline{\CRIT}_{u,t}, \\
\omega_u^{(t)} &= \frac{\Gamma_u^{(t)}+\epsilon}{\Gamma_A^{(t)}+\Gamma_B^{(t)}+2\epsilon},
\end{aligned}
\label{eq:reliability}
\end{equation}
\end{eqhalf}

where $\omega_u^{(t)}$ weights agent $u$'s beliefs. Rising $\CRIT$ indicates improving argument quality; agents producing consistently weak arguments receive less weight in the mixture.

%\vspace{-.05in}
\subsection{Dual-Dial Control with Plateau Detection}
\label{sec:framework:control}

\paragraph{Information gain.}
Normalized entropy reduction of reliability-weighted mixture:
\begin{eqhalf}
\begin{equation}
\hat{I}^{(t)} = \left[\frac{H(p^{(t-1)}) - H(p^{(t)})}{\log|\mathcal{Y}|}\right]_+,
\quad p^{(t)} = \sum_u \omega_u^{(t)} p_u^{(t)}.
\label{eq:ig}
\end{equation}
\end{eqhalf}

Falling $\hat{I}$ indicates lower returns from continued debate.

%\vspace{-.05in}
\paragraph{Plateau detection via relative progress ratios.}
Rather than fixed thresholds, we measure progress relative to remaining capacity:
\begin{eqhalf}
\begin{equation}
\begin{aligned}
r_I^{(t)} &= \frac{\text{MA}_w(\hat{I}^{(t)})}{\text{MA}_w(H(p^{(t-1)})/\log|\mathcal{Y}|)+\eta}, \\
r_D^{(t)} &= \frac{\text{MA}_w([\JS^{(t-1)} - \JS^{(t)}]_+)}{\text{MA}_w(\JS^{(t-1)})+\eta},
\end{aligned}
\label{eq:ratios}
\end{equation}
\end{eqhalf}

where $\text{MA}_w$ is length-$w$ moving average, $\eta=10^{-3}$ prevents division by zero. Binary flags trigger when ratios fall below thresholds for $\tau_{\text{stop}}$ consecutive rounds:
\begin{eqhalf}
\begin{equation}
\pi_I^{(t)} = \mathbb{1}[r_I^{(t)} < \varepsilon_I^{\text{rel}}],
\quad \pi_D^{(t)} = \mathbb{1}[r_D^{(t)} < \varepsilon_D^{\text{rel}}].
\label{eq:flags}
\end{equation}
\end{eqhalf}

%\vspace{-.05in}
\paragraph{Behavior dial: scheduled contentiousness.}
When plateaus detected, reduce contentiousness:
\begin{eqhalf}
\begin{equation}
\CL^{(t+1)} = [\CL^{(t)} - \alpha_I \pi_I^{(t)} - \alpha_D \pi_D^{(t)}]_+.
\label{eq:cl_schedule}
\end{equation}
\end{eqhalf}

This transitions agents from exploratory ($\CL=0.9$: ``challenge aggressively'') to consolidative ($\CL=0.5$: ``build on agreements'').

%\vspace{-.05in}
\paragraph{Information dial: raising evidence standards.}
As consensus forms, admit increasingly higher-quality arguments:
\begin{eqhalf}
\begin{equation}
\boldsymbol{\tau}^{(t+1)} = \min\!\big\{\boldsymbol{\tau}^{(t)} + \gamma\,\pi_I^{(t)},\, \tau_{\max}\mathbf{1}\big\},~
\boldsymbol{\tau}^{(t)} = \big(\tau_{\Q}^{(t)},\,\tau_{\CRIT}^{(t)}\big),
\label{eq:tau_schedule}
\end{equation}
\end{eqhalf}
\noindent where the minimum is applied element-wise, and $\mathbf{1}\in\mathbb{R}^m$ is the all-ones vector for $m$ gates (here $m{=}2$).

%\vspace{-.05in}
\paragraph{Stopping criterion.}
Halt when both plateaus persist and evidence quality is sufficient:
\begin{eqhalf}
\begin{equation}
\begin{aligned}
&\sum_{s=0}^{\tau_{\text{stop}}-1} \pi_I^{(t-s)} = \tau_{\text{stop}},~
\sum_{s=0}^{\tau_{\text{stop}}-1} \pi_D^{(t-s)} = \tau_{\text{stop}}, \\
&\Q^{(t)} \ge \tau_{\Q}^{(t)},\;\; \Ov^{(t)} \ge \tau_{\Ov}.
%&\Q(t) \ge \tau_{\Q}^{(t)}.
%&\Q(t) \ge \tau_{\Q}^{(t)},\;\; \O(t) \ge \tau_{\O}.
%&\Q^{(t)} \ge \tau_{\Q}, \quad \Ov^{(t)} \ge \tau_{\Ov}.
\end{aligned}
\label{eq:stop}
\end{equation}
\end{eqhalf}

This ensures that the gains have genuinely plateaued, not just temporarily stalled.

%\vspace{-.05in}
\paragraph{Learned scheduler (summary).}
We replace the fixed $(\alpha,\gamma)$ schedule with a tiny bandit policy over a discrete action set of $(\Delta\alpha,\Delta\gamma)$ moves, optimizing per-round information gain and dispersion drop under a token budget. Proposition~\ref{prop:ucb-regret} (Appx.~\ref{app:learned-scheduler}) shows $\tilde{O}(\sqrt{KT})$ no-regret versus the best fixed schedule with zero expected budget violation.

\paragraph{Plateau guarantees (summary).}
Under a gated averaging update with bounded noise, dispersion $D(t)$ is non-increasing and the plateau controller terminates in $O(1/\varepsilon)$ expected rounds. With a contraction condition the bound becomes $O(\log(1/\varepsilon))$. See Lemma~\ref{lem:monotone} and Proposition~\ref{prop:termination} (Appx.~\ref{app:theory-plateau}).

\emph{On dispersion.} Our guarantees analyze dispersion as the average $\mathrm{KL}$ to the running mixture, $D_{\mathrm{KL}}(t)=\frac{1}{m}\sum_i \mathrm{KL}(p_t^{(i)}\Vert \bar p_t)$, which yields monotonicity under gated averaging by convexity; in plots we additionally report $\mathrm{JSD}$, which is locally Lipschitz-equivalent on the $\delta$-smoothed simplex (Appx.~\ref{app:theory-plateau}).

%\vspace{-.05in}
\subsection{Initialization Strategy}
\label{sec:framework:init}

Before debate, we compute initial evidence quality $\Q$ and disagreement $\JS$ from the query and retrieved context. We initialize contentiousness as
\begin{eqhalf}
\[
\CL^{(1)} =
\begin{cases}
0.7 & \text{if } \Q \ge \tau_{\Q}^{\mathrm{hi}} \ \text{and}\ \JS \le \tau_{\JS}^{\mathrm{lo}},\\[2pt]
0.9 & \text{if } \Q \le \tau_{\Q}^{\mathrm{lo}} \ \text{and}\ \JS \ge \tau_{\JS}^{\mathrm{hi}},\\[2pt]
0.8 & \text{otherwise.}
\end{cases}
\]
\end{eqhalf}

The thresholds $\tau_{\Q}^{\mathrm{lo}},\tau_{\Q}^{\mathrm{hi}},\tau_{\JS}^{\mathrm{lo}},\tau_{\JS}^{\mathrm{hi}}$ are set by percentiles on the development set (Appendix~\ref{app:bridge}). This reduces wasted exploration on straightforward cases while preserving search capacity on ambiguous ones.

%\vspace{-.05in}
\subsection{Complete Algorithm}
\label{sec:framework:algorithm}

Algorithm~\ref{alg:maci} integrates all components (see Appx.~\ref{app:maci_algorithm} for complete pseudocode). 

%\vspace{-.05in}
\subsection{Problem regime classification}
$\MACI$ targets open-ended reasoning where (1) the answer space is large ($|\mathcal{Y}| \ge 20$), (2) evidence synthesis is required rather than pattern matching, and (3) uncertainty quantification matters. For small multiple-choice tasks ($|\mathcal{Y}| \le 5$), majority voting typically suffices. Appendix~\ref{app:regime} formalizes this distinction.

\section{Empirical Validation}
\label{sec:experiments}

We evaluate $\MACI$ in two domains, clinical diagnosis (G1) and news bias detection (G2), measuring accuracy, calibration, and cost, with ablations and case studies.

\begin{comment}
%\vspace{-.05in}
\paragraph{Reproducibility.}
Code, prompts, seeds, and run scripts are in the anonymized supplement \citep{anonymous2025supplement}. API runs fix model versions; local runs use commodity accelerators. Full configurations appear in Appendix~\ref{app:config}.
\end{comment}

%\vspace{-.05in}
\subsection{G1: Clinical Diagnosis}
\label{sec:clinical}

%\vspace{-.05in}
\paragraph{Dataset and protocol.}
\textbf{Kaggle Symptom--Disease}~\citep{KaggleDS2023}: 1{,}500 cases across 30 diseases. The corpus provides only current symptoms (no labs or history). We use the \emph{unconstrained} setting: models output free-text diagnoses from $100{+}$ possibilities (no candidate list), matching clinical triage with broad differentials.

%\vspace{-.05in}
\paragraph{Why voting fails here.}
When predictions are constrained to 30 labels, accuracy reaches 95--97\%. In the unconstrained regime, agreement collapses and Acc@1 falls to 73--78\% because many cases are under-specified and require evidence synthesis rather than pattern matching. Majority voting is unreliable in large output spaces (Appx.~\ref{app:majority}).

%\vspace{-.05in}
\paragraph{Setup.}
\textbf{Backbones:} GPT-4o, Claude~3 Opus, Gemini~1.5 Pro; all pairs ($M{=}2$). \textbf{Retrieval:} BM25 + dense reranking ($k_{\text{ret}}{=}32$). \textbf{Baselines:} best single; majority vote (20); self-consistency (20); fixed-contentiousness debate~\citep{liang-etal-2024-encouraging}. \textbf{Statistics:} BCa 95\% CIs from $B{=}20$ bootstraps; paired Wilcoxon with Benjamini--Hochberg FDR ($q{=}0.05$).

\begin{table}[t]
\caption{\small Pairwise $\MACI$ vs.\ single models on 1{,}500 cases. Stars: paired Wilcoxon vs.\ best single (BH FDR $q{=}0.05$): $^{***}p{<}0.001$, $^{**}p{<}0.01$.}
%\vspace{-.05in}
\label{tab:clinical_pairs}
\centering
\footnotesize
\setlength{\tabcolsep}{5pt}
\begin{tabular}{lccc}
\toprule
\textbf{Model} & \textbf{Acc@1} & \textbf{Acc@3} & \textbf{MRR} \\
\midrule
GPT-4o             & $0.734$ & $0.847$ & $0.781$ \\
Claude~3 Opus      & $0.720$ & $0.832$ & $0.768$ \\
Gemini~1.5 Pro     & $0.693$ & $0.825$ & $0.756$ \\
\midrule
GPT+Claude         & $\mathbf{0.786}^{***}$ & $\mathbf{0.874}^{***}$ & $\mathbf{0.823}^{***}$ \\
GPT+Gemini         & $0.751^{**}$          & $0.861^{**}$           & $0.807^{**}$ \\
Claude+Gemini      & $0.745^{***}$         & $0.858^{***}$          & $0.798^{***}$ \\
\bottomrule
\end{tabular}
%\vspace{-.1in}
\end{table}

%\vspace{-.05in}
\subsubsection{Main Results, Three Parts}
%\vspace{-.05in}
\paragraph{Results \#1: $\MACI$ vs.\ single models.}
Every $\MACI$ pair outperforms its best single (Table~\ref{tab:clinical_pairs}); for example, GPT\textendash4o+Claude improves Acc@1 by +5.2\,pp.

%\vspace{-.05in}
\paragraph{Results \#2: $\MACI$ ablation.}
$\MACI$ relies on three active controls: (i) behavioral scheduling (high $\CL$ for exploration, low $\CL$ for consolidation), 
(ii) \emph{dual gating} ($Q \ge \tau_{\text{Q}}$ \textbf{and} $\mathrm{CRIT} \ge \tau_{\text{CRIT}}$),
and (iii) information\textendash theoretic stopping (plateaus in $\JS$ and IG). Fixed\textendash contentiousness debate wastes tokens; passive aggregation uses 5--7$\times$ more tokens and underperforms. Table~\ref{tab:maci_ablations} shows orthogonal effects: removing the $\Q$ gate costs $-5.2$\,pp Acc@1 and raises ECE by $+0.040$; uniform weights (no reliability) cost $-3.0$\,pp and $+0.013$ ECE; disabling scheduling (holding $\CL$ high/low) costs $-3.9$ to $-6.0$\,pp with larger ECE.

\begin{table}[H]
\caption{\small Ablations on clinical diagnosis. $\Delta$ is the drop relative to full $\MACI$ (lower is better). (Uniform weights = no reliability weighting)}
\centering
\footnotesize
\setlength{\tabcolsep}{4pt}
\begin{tabular}{lcccc}
\toprule
Variant & Acc@1 & $\Delta$ & ECE & $\Delta$ \\
\midrule
$\MACI$ (full)        & 77.8 & --  & 0.081 & -- \\
No schedule (high)    & 73.9 & 3.9 & 0.112 & 0.031 \\
No schedule (low)     & 71.8 & 6.0 & 0.099 & 0.018 \\
No $Q$ gate           & 72.6 & 5.2 & 0.121 & 0.040 \\
Uniform weights       & 74.8 & 3.0 & 0.094 & 0.013 \\
\bottomrule
\end{tabular}
\label{tab:maci_ablations}
%\vspace{-.19in}
\end{table}

\begin{table*}[t]
\caption{\small Main results (1{,}500 cases). $\MACI$ improves accuracy and calibration with fewer tokens and rounds.}
\label{tab:maci_main_results}
%\vspace{-.05in}
\centering
\footnotesize
\setlength{\tabcolsep}{5pt}
\begin{tabular}{l cccccc}
\toprule
Method & Acc@1 & MRR & ECE & Brier & Tokens ($10^3$) & Rounds \\
\midrule
Single best           & 72.1 & 0.802 & 0.126 & 0.187 & 6.8   & 1.0 \\
Majority vote (20)    & 73.9 & 0.817 & 0.139 & 0.196 & 92.3  & 1.0 \\
Self-consistency (20) & 74.7 & 0.827 & 0.147 & 0.201 & 128.6 & 1.0 \\
Liang et al.~\citep{liang-etal-2024-encouraging} & 74.1 & 0.834 & 0.103 & 0.173 & 22.6  & 3.3 \\
\textbf{$\MACI$ (ours)}  & \textbf{77.8} & \textbf{0.856} & \textbf{0.081} & \textbf{0.158} & \textbf{18.4} & \textbf{2.7} \\
\bottomrule
\end{tabular}
%%\vspace{-.07in}
\end{table*}

\begin{figure*}[t!]
\centering
\begin{subfigure}[t]{0.33\linewidth}
  \centering
  \includegraphics[width=\linewidth,height=4.1cm]{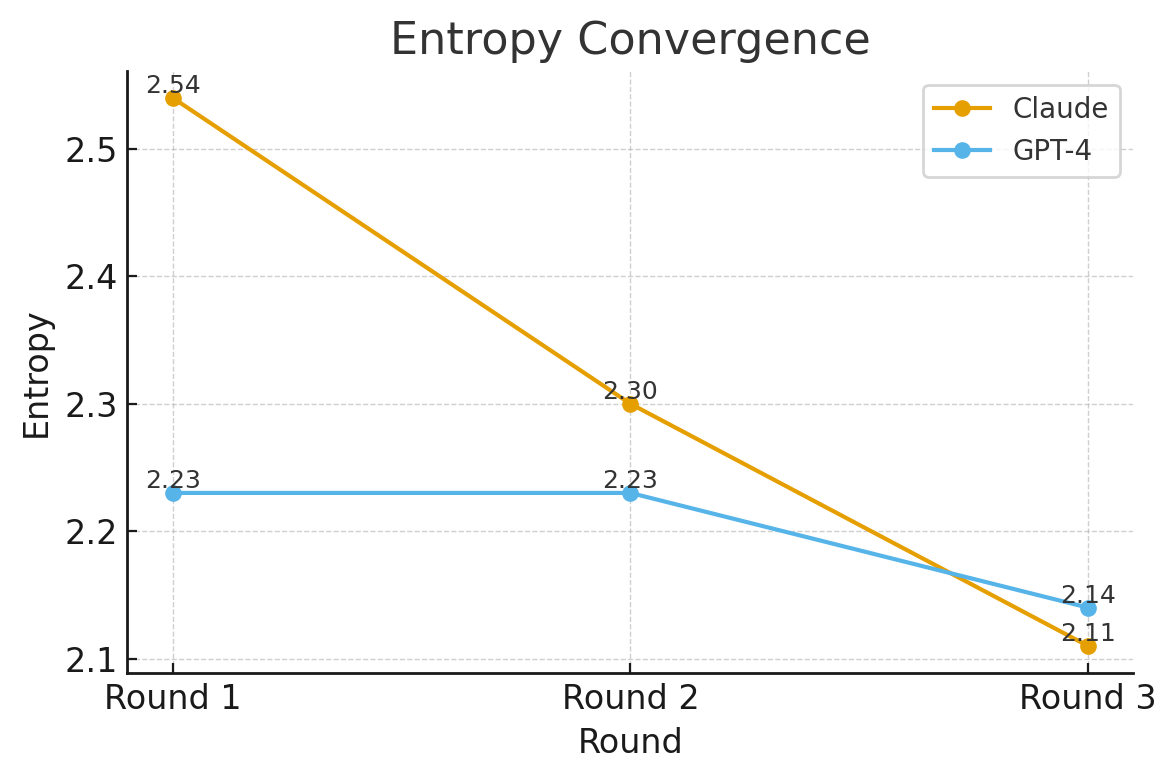}
  \caption{\footnotesize Entropy convergence.}
\end{subfigure}\hfill
\begin{subfigure}[t]{0.34\linewidth}
  \centering
  \includegraphics[width=\linewidth,height=3.9cm]{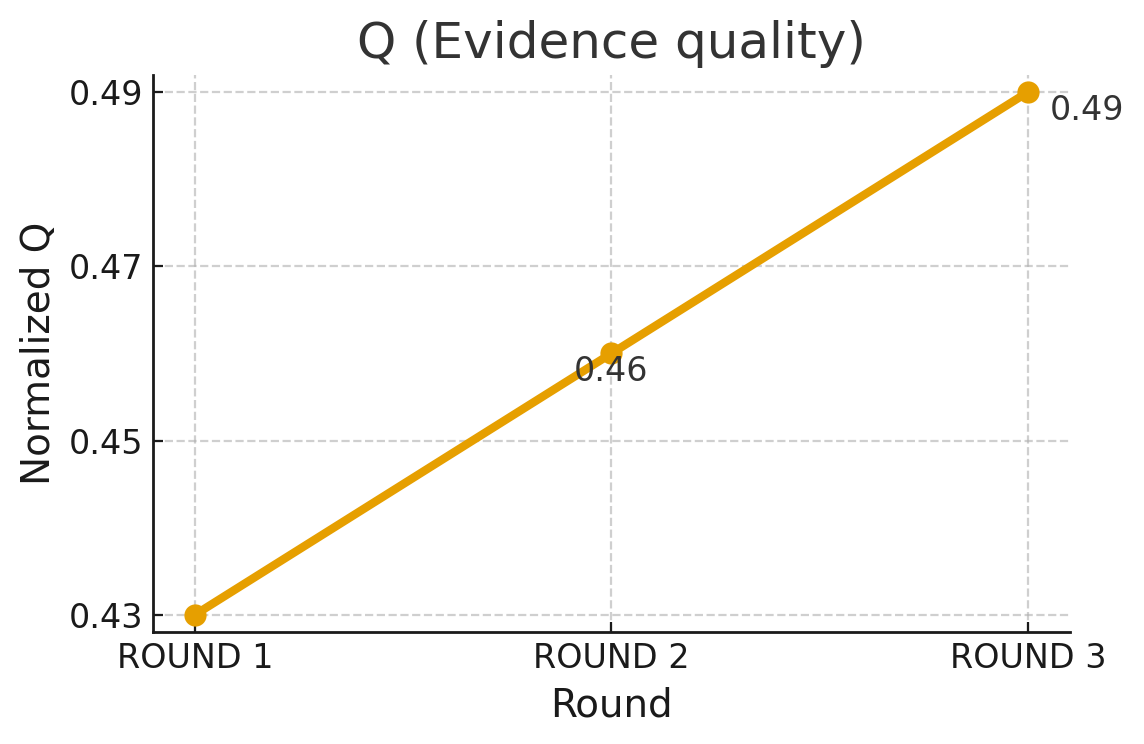}
  \caption{\footnotesize Evidence quality ($\Q$) improves.}
\end{subfigure}\hfill
\begin{subfigure}[t]{0.33\linewidth}
  \centering
  \includegraphics[width=\linewidth,height=3.6cm]{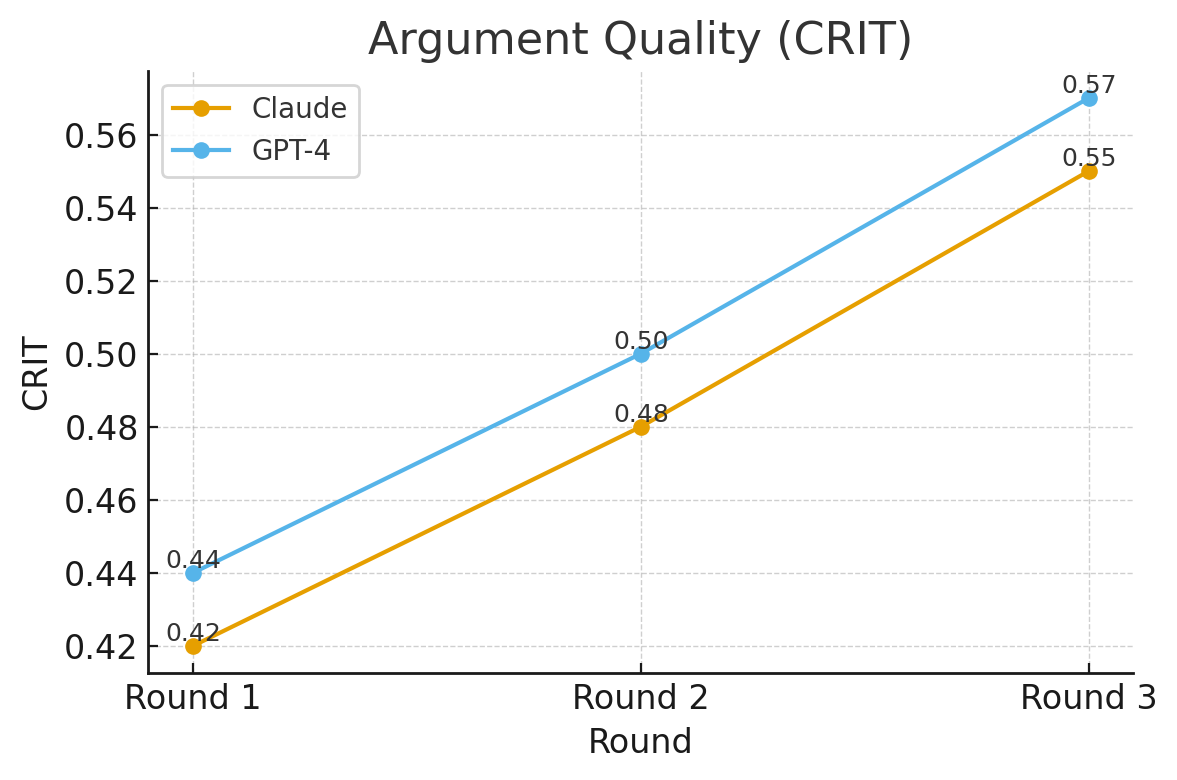}
  \caption{\footnotesize Argument quality ($\CRIT$) improves.}
\end{subfigure}
%\vspace{-0.1in}
\caption{\small Internal signals over debate rounds (clinical). \textbf{Left:} entropy declines under scheduled $\CL$. 
\textbf{Middle:} evidence quality $\Q$ rises as the evidence gate $\tau_{\Q}$ tightens.
\textbf{Right:} argument quality $\CRIT$ rises as low-quality arguments are filtered. Termination coincides with plateaued IG and low dispersion (not shown).}
\label{fig:entropy_q_crit}
%%\vspace{-.1in}
\end{figure*}

%\vspace{-.05in}
\paragraph{Results \#3: $\MACI$ vs.\ competing schemes.}
Against debate/aggregation baselines, $\MACI$ is more accurate and better calibrated with fewer tokens (Table~\ref{tab:maci_main_results}): +3.9\,pp over majority vote and +3.7\,pp over Liang et al.\ (both $p{<}0.001$), ECE $0.081$ vs.\ $0.103$, and $19\%$ fewer generation tokens. These gains align with the ablation evidence.

%\vspace{-.05in}
\paragraph{Computational overhead.}
$\MACI$ uses 18.4k tokens vs.\ 6.8k for a single model (2.7$\times$ overhead). For high\textendash stakes settings where accuracy reduces costly errors, this overhead is justified; for routine tasks, single\textendash model inference may suffice.

%\vspace{-.05in}
\paragraph{Evidence quality stratification.}
Accuracy correlates with initial $\Q$: cases with $\Q{>}0.7$ reach 78.3\% vs.\ 69.8\% for $\Q{<}0.3$ (Table 20, Appx.~K). This validates $\Q$ as a useful signal and suggests retrieval improvements may benefit $\MACI$ more than parameter tuning.

%\vspace{-.05in}
\paragraph{Controller guarantee (theory\textendash lite).}
With bounded noise and gated averaging, dispersion $D(t)$ is non\textendash increasing, and the plateau controller terminates in $O(1/\varepsilon)$ expected rounds; under a contraction, $D(t)$ decays geometrically and termination is $O(\log(1/\varepsilon))$ (Lem.~\ref{lem:monotone}, Prop.~\ref{prop:termination}, Cor.~\ref{cor:geometric}, Appx.~\ref{app:theory-plateau}).

%\vspace{-.05in}
\subsubsection{Two Case Studies: Active Control and Information Acquisition}

We illustrate $\MACI$’s scheduling and gating, and show termination aligns with plateaued IG and low dispersion. Each case ends with a targeted information\textendash acquisition plan specifying \emph{what to fetch next} for retrieval\textendash augmented generation. Full transcripts appear in Appx.~\ref{app:maci_transcript} and \ref{app:maci_hepatitis}.

%\vspace{-.05in}
\paragraph{Case 1: Dengue vs.\ Chikungunya (3 rounds).}
\textbf{Symptoms:} skin rash, joint pain, vomiting, fatigue, high fever, headache, retro\textendash orbital pain, myalgia, red spots. \textbf{Ground truth:} Dengue Fever.

Table~\ref{tab:dengue_signals} shows scheduling $\CL \in \{0.9,0.7,0.5\}$. \emph{R1} ($\CL{=}0.9$): GPT\textendash4o proposes specific viral diagnoses (Dengue 60\%, Chikungunya 25\%, Zika 15\%); Gemini stays categorical (Viral Infection 60\%, Autoimmune 20\%, Bacterial 15\%). High disagreement ($\JS{=}0.56$). \emph{R2} ($\CL{=}0.7$): targeted challenges shift Gemini to Dengue (50\%) and Chikungunya (30\%); $\JS{=}0.37$, $\CRIT$ 0.75\,$\rightarrow$\,0.82. \emph{R3} ($\CL{=}0.5$): distributions align (Dengue 60\%, Chikungunya 35\%, Zika 5\%); $\JS{=}0.02$, MI plateaus (0.49); trigger stop flags.

\begin{table}[t]
\centering
\caption{\small Signal evolution on a Dengue case. $\widetilde{W}_1$: Wasserstein-1 distance; $\mathrm{MI}$: mutual information; $\JS$: Jensen--Shannon divergence. Lower is better for $\widetilde{W}_1$ and $\JS$; higher is better for $\mathrm{MI}$. (Appx~\ref{app:signals})}
\label{tab:dengue_signals}
\footnotesize
\begin{tabular}{lcccc}
\toprule
\textbf{Round} & $\boldsymbol{\CL}$ & $\boldsymbol{\widetilde{W}}_{\boldsymbol{1}}$ & $\boldsymbol{\mathrm{MI}}$ & $\boldsymbol{\JS}$ \\
\midrule
1 & 0.9 & 0.68 & 0.43 & 0.56 \\
2 & 0.7 & 0.44 & 0.46 & 0.37 \\
3 & 0.5 & 0.36 & 0.49 & 0.02 \\
\midrule
\multicolumn{2}{l}{\textit{Change (R1 $\to$ R3)}} & $-47\%$ & $+14\%$ & $-96\%$ \\
\bottomrule
\end{tabular}
%\vspace{-.05in}
\end{table}

%\vspace{-.06in}
\paragraph{Figure 1 (internal signals).}
Fig.~\ref{fig:entropy_q_crit} shows entropy declining as $\CL$ moves from explore to consolidate; $\Q$ rising as the gate $\tau$ tightens; and $\CRIT$ increasing as low\textendash quality arguments are filtered. Dispersion $\JS$ (not shown) falls in parallel; termination matches IG and dispersion plateaus.

\paragraph{Evaluator robustness.}
Reversing debater order leaves the $\CRIT$ winner unchanged; cross\textendash family judge swaps preserve item\textendash level rankings (Appx.~\ref{app:crit-robustness}).

\paragraph{Informative RAG recommendations.}
Beyond “Dengue Fever (60\%),” $\MACI$ outputs a plan: \emph{Questions} (endemic travel; bleeding/bruising), \emph{Labs} (CBC for thrombocytopenia; IgM/IgG; PCR; NS1 antigen), plus retrieval cues for RAG.

%\vspace{-.05in}
\paragraph{Case 2: Hepatitis differential (4 rounds, adaptive stopping).}
\textbf{Symptoms:} itching, vomiting, fatigue, weight loss, high fever, jaundice, dark urine, abdominal pain. \textbf{Ground truth:} Hepatitis~C.

\begin{figure*}[t]
\centering
\includegraphics[width=.47\textwidth,height=3.9cm]{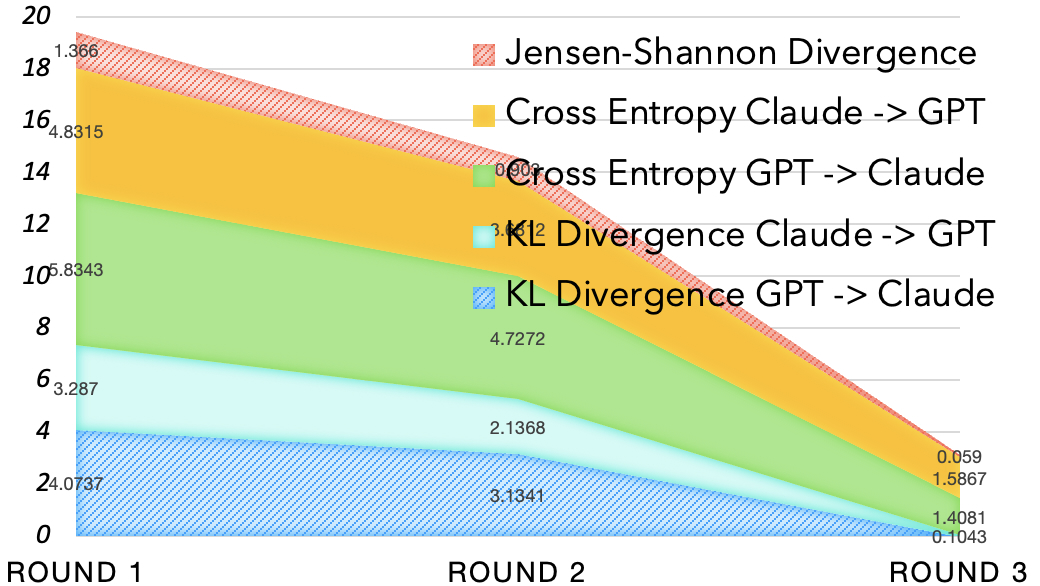}\hspace{-.01in}
\includegraphics[width=.47\textwidth,height=4.1cm]{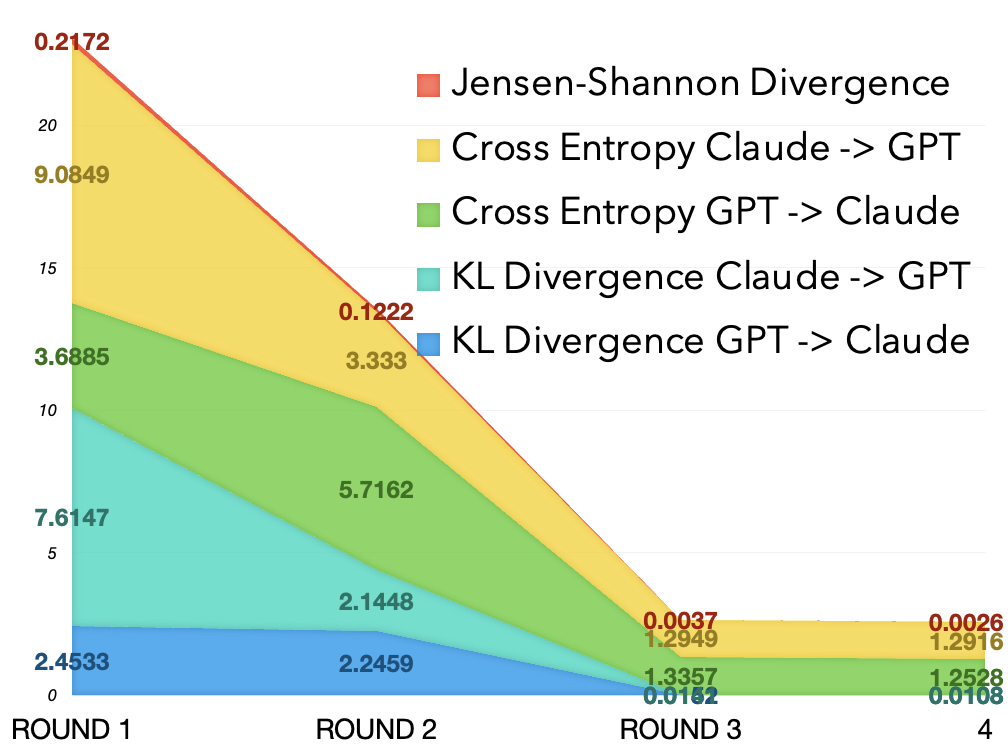}
\caption{Convergence in two cases: $D_{\mathrm{JS}}$ and allied distances decrease monotonically under scheduling.}
\label{fig:case_convergence}
%%\vspace{-.08in}
\end{figure*}

Table~\ref{tab:hepatitis_signals} shows $\CL \in \{0.9,0.7,0.5,0.3\}$. \emph{R1} ($\CL{=}0.9$): GPT\textendash4o emphasizes HCV (40\%); Claude emphasizes HBV (35\%) and alcoholic hepatitis (15\%); $\JS{=}0.22$. \emph{R2} ($\CL{=}0.7$): reasoning integration; $\JS{=}0.12$, $\CRIT$ rises. \emph{R3} ($\CL{=}0.5$): near convergence (HCV 35\%, HBV 25–30\%, cirrhosis 20\%); $\JS{=}0.04$, MI 0.49. \emph{R4} ($\CL{=}0.3$): distributions nearly identical; termination.

\begin{table}[H]
\centering
\caption{\small Signal evolution: Hepatitis case (4 rounds).}
\label{tab:hepatitis_signals}
\footnotesize
\begin{tabular}{cccccc}
\toprule
Round & $\CL$ & $\widetilde{W}_1$ & MI & $\CRIT$ & $\JS$ \\
\midrule
1 & 0.9 & 0.52 & 0.39 & 0.76 & 0.22 \\
2 & 0.7 & 0.45 & 0.41 & 0.83 & 0.12 \\
3 & 0.5 & 0.33 & 0.49 & 0.89 & 0.04 \\
4 & 0.3 & 0.31 & 0.49 & 0.92 & 0.03 \\
\midrule
\multicolumn{2}{l}{Change} & $-40\%$ & $+26\%$ & $+21\%$ & $-86\%$ \\
\bottomrule
\end{tabular}
%\vspace{-.1in}
\end{table}

%\vspace{-.07in}
\paragraph{Internal signals (case view).}
As in Case 1, entropy declines as $\CL$ schedules down; $\Q$ and $\CRIT$ rise with tighter gates; $\JS$ falls toward zero. Fig.~\ref{fig:case_convergence} visualizes monotonic dispersion declines for both cases.

%\vspace{-.06in}
\paragraph{Evaluator robustness.}
The same checks apply: order swap and cross\textendash family judge swap leave outcomes stable (Appx.~\ref{app:crit-robustness}).

%\vspace{-.06in}
\paragraph{Informative RAG recommendations.}
$\MACI$ outputs: \emph{Labs} (HAV/HBV/HCV serology; LFTs ALT/AST/bilirubin/alkaline phosphatase; CBC; INR), \emph{Imaging} (abdominal ultrasound), \emph{History} (alcohol/medications, injection drug use, transfusions), \emph{Monitoring} (encephalopathy, coagulopathy).

%%%%%%%%% G1 done

%\vspace{-.05in}
\subsection{G2: News Bias Detection (Cross\textendash Domain)}
\label{sec:news_bias}

%\vspace{-.05in}
\paragraph{Why this tests portability.}
News bias differs from G1 on \emph{evidence} (full articles vs.\ symptom lists), \emph{labels} (5\textendash point continuous vs.\ discrete), and \emph{ground truth} (partisan\textendash anchored vs.\ objective). A self\textendash supervised study maps $\CL$ to linguistic patterns across topics (Table~\ref{tab:cont-linguistics}, Appx.~\ref{app:contentiousness}): high $\CL{=}0.9$ is confrontational, low $\CL{=}0.3$ is cooperative. A shallow classifier predicts $\CL$ bins with AUC $0.83$ (Table~\ref{tab:cl-aucs}), indicating measurable behavior modulation beyond domain content.

%\vspace{-.05in}
\paragraph{Dataset and setup.}
We use 619 scandal articles from \citep{Budak2016} with Democrat/Republican five\textendash point ratings ($-2$ to $+2$). Hyperparameters match G1: $\CL \in \{0.9,0.7,0.5\}$, gates ($\tau_{\Q}$, $\tau_{\Ov}$), and the same stopping rule. Two agents (GPT\textendash4o, Claude~3 Opus) start with asymmetric partisan framings.

\begin{table}[th]
\caption{\small News bias results (619 articles). \emph{Tokens are reported in thousands} ($\times 10^3$). MAD: mean absolute deviation from centrist (0).}
%\vspace{-.05in}
\centering
\footnotesize
\setlength{\tabcolsep}{4pt}
\begin{tabular}{lccccc}
\toprule
Method & MAD & ECE & Brier & Tokens & Rounds \\
\midrule
D\textendash annotator      & 0.84 & 0.168 & 0.221 & --   & -- \\
R\textendash annotator      & 0.76 & 0.173 & 0.218 & --   & -- \\
Average (D+R) & 0.51 & 0.142 & 0.189 & --   & -- \\
Unscheduled    & 0.38 & 0.127 & 0.178 & 11.7 & 3.8 \\
\textbf{$\MACI$} & \textbf{0.12} & \textbf{0.093} & \textbf{0.167} & \textbf{8.2} & \textbf{2.4} \\
\bottomrule
\end{tabular}
\label{tab:news_main_results}
%\vspace{-.05in}
\end{table}

\begin{comment}
\begin{figure}[t]
\centering
\begin{subfigure}[t]{0.49\linewidth}
  \centering
  \includegraphics[width=\linewidth,height=3.6cm]{DR-Bias-Ratings.png}
  \caption{\footnotesize Partisan vs.\ MACI distributions.}
\end{subfigure}\hfill
\begin{subfigure}[t]{0.49\linewidth}
  \centering
  \includegraphics[width=\linewidth,height=3.6cm]{InformationMetrics.jpg}
  \caption{\footnotesize Convergence signals over rounds.}
\end{subfigure}
%\vspace{-.05in}
\caption{Partisan rating distributions on scandal articles. Democrat-leaning raters (D) and Republican-leaning raters (R) show partisan peaks. $\MACI$ centers between them, reducing MAD by 68\% versus unscheduled debate.}
\label{fig:news_main}
%\vspace{-.08in}
\end{figure}
\end{comment}

%%\vspace{-.05in}
\paragraph{Results \#1: Narrowing partisan gaps.}
Partisan annotators diverge (MAD 0.84 vs.\ 0.76). Averaging yields 0.51; unscheduled debate 0.38. $\MACI$ achieves \textbf{0.12} MAD (\textbf{76\%} vs.\ averaging, \textbf{68\%} vs.\ unscheduled), with better calibration (ECE 0.093 vs.\ 0.127–0.173) and \textbf{30\%} fewer tokens (8.2 vs.\ 11.7). Fig.~\ref{fig:news_main}a shows centering between partisan peaks.

%%\vspace{-.05in}
\paragraph{Results \#2: Which controls matter (ablations).}
Removing scheduling raises MAD by {+}0.14–0.19; removing the $\Q$ gate by {+}0.17; uniform (no\textendash reliability) weights by {+}0.06 and ECE by {+}0.014 (Table~\ref{tab:news_ablations}).

\begin{table}[th]
\caption{\small Ablations. Positive $\Delta$ MAD indicates worse bias.}
%\vspace{-.05in}
\centering
\footnotesize
\setlength{\tabcolsep}{3.5pt}
\begin{tabular}{lcccc}
\toprule
Variant & MAD & $\Delta$ & ECE & $\Delta$ \\
\midrule
$\MACI$ (full)       & 0.12 & --   & 0.093 & -- \\
No schedule (high)   & 0.31 & 0.19 & 0.134 & 0.041 \\
No schedule (low)    & 0.26 & 0.14 & 0.118 & 0.025 \\
No $\Q$ gate         & 0.29 & 0.17 & 0.141 & 0.048 \\
Uniform weights      & 0.18 & 0.06 & 0.107 & 0.014 \\
\bottomrule
\end{tabular}
\label{tab:news_ablations}
%\vspace{-.07in}
\end{table}

%%\vspace{-.05in}
\paragraph{Results \#3: Signal convergence and behavioral transfer.}
Debates average $2.4$ rounds and stop when dispersion and information gains plateau. Convergence traces mirror G1; see Appx.~\ref{app:news_bias} for signal plots.

\begin{figure}[t]
\centering
\includegraphics[width=.95\linewidth,height=5.3cm]{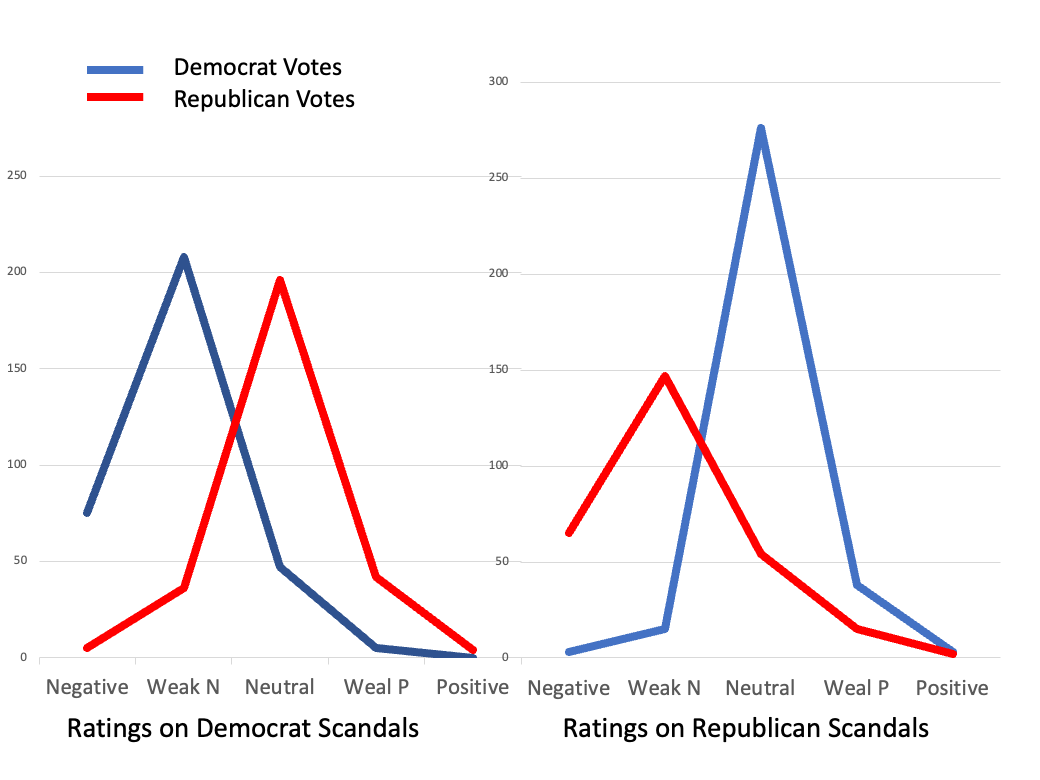}
%%\vspace{-.05in}
\caption{\small Convergence signals during bias mitigation (news bias). Wasserstein distance falls, normalized MI rises and then plateaus, and cross-entropy declines. Debates stop when dispersion and information gains plateau, mirroring the clinical setting (Appx.~\ref{app:news_bias}).}
\label{fig:news_main}
%%\vspace{-.08in}
\end{figure}

%\vspace{-.05in}
\paragraph{Example (one case).}
On “Christopher Columbus,” one agent begins Euro\textendash centric (expansion and trade) and the other foregrounds indigenous harms. Over four rounds the former acknowledges severe consequences within broader colonial policy and recommends balanced presentation with citations (Table~\ref{tab:debate_arguments}, Appx.~\ref{app:news_bias}).

%\vspace{-.05in}
\paragraph{Portability.}
Same code and hyperparameters, no domain tuning. Controls ($\CRIT$ and $\Q$ gates, $\CL$ scheduling, $\JS$ and $\Ov$ stopping) transfer across evidence types and label structures. Per\textendash article distributions appear in Appx.~\ref{app:news_bias}.

%\vspace{-.05in}
\paragraph{Budget-aware learned scheduler.}
We also provide a drop\textendash in scheduler that replaces the fixed $(\alpha,\gamma)$ schedule with a small bandit over $(\Delta\alpha,\Delta\gamma)$ moves, optimizing per\textendash round IG gains and dispersion drops under a token budget. The policy enjoys $\tilde{O}(\sqrt{KT})$ no\textendash regret against the best fixed schedule with zero expected budget violation (Prop.~\ref{prop:ucb-regret}, Appx.~\ref{app:learned-scheduler}). We keep the fixed schedule in G1/G2 main tables for comparability.

%\emph{Threats to validity.} Budak et al.’s 2013 annotations may not reflect current discourse; we mitigate judge bias via cross–family $\CRIT$ evaluators and order–invariance checks (Appx.~\ref{app:crit-robustness}), but broader outlet coverage and newer corpora are natural next steps.

%\input{ExperimentMACI}
%\vspace{-.05in}
\section{Conclusion}
\label{sec:conclusion}

%\vspace{-.05in}
\paragraph{Advancing multi-agent orchestration.}
We presented $\MACI$, a control framework that extends contentiousness modulation~\cite{Chang2023CRIT,SocraSynthChangCSCI2023} with (1) \emph{dual-dial} control separating information admission from behavioral stance, (2) \emph{information-theoretic} stopping via relative progress ratios, and (3) \emph{adaptive} initialization to reduce wasted exploration. Unlike passive aggregation or fixed debate, $\MACI$ steers deliberation using measurable signals—disagreement, overlap, evidence quality, and argument quality—and halts when gains plateau.

On clinical diagnosis (1{,}500 cases), $\MACI$ improves accuracy (+3.9\,pp over majority vote; +3.7\,pp over fixed-contentiousness) and calibration (ECE $0.081$ vs.\ $0.103$) while using $19\%$ fewer generation tokens; case studies show convergence (disagreement $\downarrow$86--96\%, $\CRIT$ $\uparrow$21--26\%). $\CRIT$ is order-invariant and stable under cross-family judge swaps (2--3\% winner flips; Appx.~\ref{app:crit-robustness}). Theory-lite guarantees bound dispersion and ensure termination in $O(1/\varepsilon)$ rounds (improving to $O(\log(1/\varepsilon))$), and a budget-feasible UCB scheduler attains $\tilde{O}(\sqrt{KT})$ no-regret with zero expected budget violation (Appx.~\ref{app:theory-plateau},~\ref{app:learned-scheduler}). News-bias validation (619 articles) shows portability: the same controls reduce the partisan gap by $68\%$ without domain tuning.

%\vspace{-.05in}
\paragraph{Limitations and directions.}
\begin{itemize}[leftmargin=1.0em,itemsep=-1pt,topsep=0pt]
    \item \textbf{Scope.} We evaluate two domains with subjective judgment; broader testing on tasks with objective ground truth (math, code, factual QA) and longer chains remains future work. Retrieval quality mediates performance (Sec.~\ref{sec:clinical}), suggesting gains from stronger evidence acquisition.
    \item \textbf{LLM-judge dependence.} $\CRIT$ uses cross-family evaluators and may inherit bias. We mitigate via identity masking, span-grounded justifications, and judge-swap checks (2--3\% flips; Appx.~\ref{app:crit-robustness}). Lower-capability regimes may require larger $K$, tighter stop tolerances, and periodic human audits.
    \item \textbf{Cost trade-offs.} $\MACI$ uses $2.7{\times}$ tokens vs.\ single-model inference; while accuracy gains justify overhead for complex cases, domain-specific cost–benefit studies are needed.
    \item \textbf{Hyperparameters.} $\CL$ schedules, gates, and plateau tolerances are tuned on dev data; generalization beyond diagnosis/news is open. Scaling to 3--5 agents and meta-learned initialization are natural next steps.
\end{itemize}
\emph{Closed-loop RAG.} A promising direction is coupling $\MACI$’s acquisition plans to budgeted RAG with end-to-end evaluation; transcripts (Appx.~\ref{app:maci_transcript}, \ref{app:maci_hepatitis}) enable auditability.

%\vspace{0in}
Overall, $\MACI$ reframes multi-agent debate as a controllable, auditable, and \emph{provably terminating} controller, turning orchestration from heuristics into a budget-aware, measurable process.

\bibliography{EdwardChang,Evince,UCCT,MACI,References-1,References-2,SocraHealth,RobustR,TSP,TemporalPlanning,QAExperiments}
\bibliographystyle{abbrvnat} 

\clearpage
\appendix
\section*{Appendix Overview}

\begin{itemize}[leftmargin=1.2em,itemsep=0pt]
\item \textbf{Appendix A}: Notation and symbols
\item \textbf{Appendix B}: Configuration, hyperparameters, experimental setup
\item \textbf{Appendix C}: $\MACI$ algorithm specification 
\item \textbf{Appendix D}: Information-theoretic signals 
\item \textbf{Appendix E}: CRIT evaluator details
\item \textbf{Appendix F}: Adaptive initialization strategy
\item \textbf{Appendix G}: Dengue case study (full transcript)
\item \textbf{Appendix H}: Hepatitis case study (full transcript)
\item \textbf{Appendix I}: Linguistic analysis and behavioral validation
\item \textbf{Appendix J}: News bias detection 
\item \textbf{Appendix K}: Majority vote failure analysis
\item \textbf{Appendix L}: Learned scheduler for $(\alpha,\gamma)$ (algorithm and no-regret analysis)
\item \textbf{Appendix M}: Plateau guarantees (monotone dispersion and termination)
%\item \textbf{Appendix N}: Judge robustness and order invariance
\end{itemize}

\thispagestyle{empty}

\section{Notation and Symbols}
\label{app:notation}

\begin{table}[H]
\caption{Key symbols used throughout.}
\centering
\footnotesize
\setlength{\tabcolsep}{4pt}
\renewcommand{\arraystretch}{1.05}
\begin{tabular}{p{0.15\linewidth} p{0.77\linewidth}}
\toprule
\textbf{Symbol} & \textbf{Meaning} \\
\midrule
$\MACI$ & Multi-Agent Collaborative Intelligence (dual-dial control framework) \\
$\Q$ & Target-aligned evidence quality \\
$\JS$ & Inter-agent disagreement (Jensen-Shannon divergence) \\
$\Ov$ & Support overlap (Jaccard over cited spans) \\
$\hat{I}$ & Normalized information gain (entropy drop) \\
$\CRIT$ & Argument-quality score from independent LLMs \\
$\CL$ & Behavior dial (contentiousness level) \\
$\tau_{\Q}$ & Evidence-quality gate (information dial)\\
$\tau_{\CRIT}$ & Argument-quality gate (evaluator dial)\\
$\tau_{\Ov}$ & Overlap floor used only in the stopping rule \\
$\omega_u^{(t)}$ & Reliability weights (EMA of $\CRIT$), agent $u$ \\
$p_u^{(t)}$ & Agent $u$'s probability distribution at round $t$ \\
$p^{(t)}$ & Reliability-weighted mixture distribution \\
$\mathcal{E}_u^{(t)}$ & Evidence spans cited by agent $u$ at round $t$ \\
$r_I^{(t)}, r_D^{(t)}$ & Relative progress ratios (information gain, disagreement) \\
$\pi_I^{(t)}, \pi_D^{(t)}$ & Plateau detection flags \\
$\widetilde{W}_1$ & Normalized Wasserstein-1 distance \\
\bottomrule
\end{tabular}
\end{table}
% =========================================================
% Appendix B: Configuration and Hyperparameters
% =========================================================
\section{Configuration and Hyperparameters}
\label{app:config}

\subsection{Default Hyperparameters}

Unless specified otherwise, $\MACI$ uses these settings across all experiments:

\paragraph{Control parameters.}
\begin{itemize}[leftmargin=1.2em,itemsep=2pt]
  \item \textbf{Reliability smoothing (EMA):} $\lambda = 0.8$, $\epsilon = 10^{-6}$.
  \item \textbf{Initial gates:} $\tau_{\Q}^{(1)} = 0.50$, $\tau_{\CRIT}^{(1)} = 0.30$.
  \item \textbf{Gate tightening:} $(\tau_{\Q}^{(t+1)},\tau_{\CRIT}^{(t+1)}) \leftarrow 
        \min\!\big\{(\tau_{\Q}^{(t)},\tau_{\CRIT}^{(t)}) + \gamma\,\pi_I^{(t)},\,(\tau_{\max},\tau_{\max})\big\}$.
  \item \textbf{Overlap floor for stopping (not an admission gate):} $\tau_{\Ov}=0.30$.
\end{itemize}

\paragraph{Initialization bins (development set).}
\begin{itemize}[leftmargin=1.2em,itemsep=1pt]
  \item High quality ($\Q \ge 0.7$) and low disagreement ($\JS \le 0.3$): $\CL^{(1)}=0.7$.
  \item Low quality ($\Q \le 0.5$) and high disagreement ($\JS \ge 0.5$): $\CL^{(1)}=0.9$.
  \item Otherwise: $\CL^{(1)}=0.8$.
\end{itemize}

\paragraph{Contentiousness schedule.}
Standard progression: $\CL \in \{0.9, 0.7, 0.5, 0.3\}$ with adaptive transitions triggered by plateau flags.

\subsection{Clinical Diagnosis (G1) Setup}

\paragraph{Dataset.}
Kaggle Symptom-Disease Corpus, stratified split: 1,500 test cases (50 per disease, 30 conditions). Train/dev/test files and random seeds provided in supplement.

\paragraph{Retrieval.}
BM25 with dense reranker (all-MiniLM-L6-v2), $k_{\text{ret}}=32$, max context 4k tokens.

\paragraph{Backbones.}
\begin{itemize}[leftmargin=1.2em,itemsep=1pt]
\item GPT-4o (gpt-4o-2024-08-06)
\item Claude 3 Opus (claude-3-opus-20240229)
\item Gemini 1.5 Pro (gemini-1.5-pro-001)
\item Temperature: 0.7 for all agents
\end{itemize}

\paragraph{Evaluation.}
20 bootstrap resamples for confidence intervals. BCa method. Paired Wilcoxon tests with Benjamini-Hochberg FDR correction at $q=0.05$.

\subsection{News Bias (G2) Setup}

\paragraph{Dataset.}
619 scandal-focused articles with partisan annotations (5-point scale). \citep{Budak2016} collected 803,146 articles from fifteen U.S. news outlets,cidentifying 10,502 political articles. We focus on scandals as this content 
exhibits the strongest partisan disagreement.

\paragraph{Agents.}
GPT-4o and Claude 3 Opus with asymmetric initialization (one Democrat-framed, one Republican-framed).

\paragraph{Identical hyperparameters.}
Same $\CL$ schedule, gates, stopping criteria as G1. No domain-specific tuning.

\section{MACI Algorithm Specification}
\label{app:maci_algorithm}

\begin{algorithm}[H]
\caption{\small $\MACI$ Moderator with Dual-Dial Control}
\label{alg:maci}
\toesize
\begin{algorithmic}[1]

\STATE Initialize debate record $\psi^{(1)}$; set $CL^{(1)}$ via dev binning; set gates $\tau_{\text{Q}}^{(1)} = 0.5$, $\tau_{\text{CRIT}}^{(1)} = 0.3$.
\FOR{$t=1,2,\dots$}
  \STATE Provide $A^{(t)} = \big(\psi^{(t)},\, CL^{(t)},\, \tau_{\text{Q}}^{(t)},\, \tau_{\text{CRIT}}^{(t)}\big)$ to both agents.
  \FOR{$u\in\{A,B\}$}
    \STATE Agent $u$ emits claim–justification pairs, $p_u^{(t)} \in \Delta(\mathcal{Y})$, cites $\hat{\mathcal{E}}_u^{(t)}$.
  \ENDFOR
  \STATE \textit{Admission (dual gate).} For each candidate argument $i$ from either agent, compute a per-argument evidence score $q_i^{(t)}$ (analogous to Eq.~\eqref{eq:q}, using only spans cited by $i$) and its $\CRIT_i^{(t)}$. 
  \STATE Admit $i$ into $\psi^{(t+1)}$ \textbf{iff} $q_i^{(t)} \ge \tau_{\text{Q}}^{(t)}$ \textbf{and} $\CRIT_i^{(t)} \ge \tau_{\text{CRIT}}^{(t)}$; otherwise discard.
  \STATE Update reliability: $\Gamma_u^{(t)} \leftarrow \lambda\,\Gamma_u^{(t-1)}+(1-\lambda)\,\overline{\CRIT}_{u,t}$; compute $\omega_u^{(t)}$ via Eq.~\eqref{eq:reliability}.
  \STATE Form mixture: $p^{(t)}=\sum_u \omega_u^{(t)} p_u^{(t)}$; compute $\hat{I}^{(t)}$ via Eq.~\eqref{eq:ig}.
  \STATE Compute signals: $\JS^{(t)}$ (Eq.~\ref{eq:jsd}), $\Ov^{(t)}$ (Eq.~\ref{eq:overlap}), $\Q^{(t)}$ (Eq.~\ref{eq:q}).
  \STATE Update progress ratios $r_I^{(t)}$, $r_D^{(t)}$ and flags $\pi_I^{(t)}$, $\pi_D^{(t)}$ via Eqs.~\eqref{eq:ratios}--\eqref{eq:flags}.
  \STATE Update dials: $CL^{(t+1)}$ via Eq.~\eqref{eq:cl_schedule}; apply Eq.~\eqref{eq:tau_schedule} \emph{componentwise} to obtain $\{\tau_{\text{Q}}^{(t+1)},\, \tau_{\text{CRIT}}^{(t+1)}\}$.
  \IF{stopping criterion Eq.~\eqref{eq:stop} satisfied}
  \STATE \textbf{return} calibrated $p^{(t)}$ and curated debate record  $\psi^{(t)}$.
  \ENDIF
\ENDFOR
\end{algorithmic}
\end{algorithm}

Default hyperparameters: $w{=}3$, $\eta{=}10^{-3}$, $\lambda{=}0.8$, $\epsilon{=}10^{-6}$, $\tau_{\text{stop}}{=}2$, $\alpha_I{=}\alpha_D{=}0.2$, $\gamma{=}0.1$, $\tau_{\text{Q}}^{(1)}{=}0.50$, $\tau_{\text{CRIT}}^{(1)}{=}0.30$.

\section{Information-Theoretic Signals for Debate Evaluation}
\label{app:signals}

This appendix presents the mathematical foundations underlying MACI's measurement and control framework. We detail each metric's properties, computational requirements, and specific role in debate moderation.

%Learned Scheduler for \texorpdfstring{$(\alpha,\gamma)$}{(alpha,gamma)}

\subsection{Target prototype construction \texorpdfstring{$(\theta)$}{{theta}}}
\label{app:theta}
% How theta is built + encoder choices + robustness notes
% Optionally: reference Eq.~\eqref{eq:q} from the main text

\subsection{Metric Selection and Usage in MACI}

We highlight each metric's properties and how MACI employs them in practice:

\begin{enumerate}[leftmargin=1.2em,itemsep=2pt]
\item \textbf{Cross-Entropy (CE).}
Likelihood surrogate that mixes uncertainty and mismatch \citep{Shore1980AxiomaticDO,Cover2006}. \emph{MACI use:} reporting and ablation studies, calibration assessment---not a primary control signal due to asymmetry and sensitivity to probability shifts.

\item \textbf{Entropy.}
Quantifies uncertainty in a single prediction \citep{shannon1948}. \emph{MACI use:} per-round \emph{information gain} via normalized entropy reduction of the reliability-weighted aggregate.

\item \textbf{Kullback--Leibler (KL) divergence.}
Directed discrepancy measure that is asymmetric and support-sensitive \citep{kullback1951information,Cover2006}. \emph{MACI use:} auxiliary analysis to detect directional belief shifts; not the primary disagreement metric due to asymmetry concerns in symmetric debate settings.

\item \textbf{Jensen--Shannon ($\JS$) divergence.}
Symmetric, bounded [0,1], and finite even under disjoint support; $\sqrt{\JS}$ satisfies triangle inequality, forming a true metric \citep{lin1991divergence}. \emph{MACI use:} primary \emph{disagreement} signal for discrete labels via unified distance measure $D^{(t)}$.

\item \textbf{Mutual Information (MI) and Normalized MI.}
Measures shared information between random variables \citep{Cover2006}; text-based MI estimators can exhibit high variance. \emph{MACI use:} logged for analysis and visualization; excluded from control loop due to estimation instability.

\item \textbf{Wasserstein Distance (WD).}
Geometry-aware discrepancy that respects label structure; efficiently computable in 1D via cumulative distribution functions \citep{kantorovich1942translocation}. \emph{MACI use:} disagreement measurement for ordered or structured outputs via unified distance measure $D^{(t)}$.

\item \textbf{Information Gain (IG).}
Entropy reduction of the reliability-weighted mixture, normalized by $\log|\mathcal{Y}|$ for interpretability \citep{Cover2006}. \emph{MACI use:} primary convergence signal with moving average plateau detection.
\end{enumerate}

\subsection{Integration in MACI Framework}

MACI employs these signals in three distinct capacities:

\paragraph{Primary Control Signals.} The unified distance measure $D^{(t)}$ adapts to output type: $D^{(t)} := 
\JS(\{B_i^{(t)}\})$ for categorical labels and $D^{(t)} := \widetilde{W}_1(\{B_i^{(t)}\})$ for structured outputs. This provides consistent plateau detection logic while respecting task-appropriate geometry.

\paragraph{Progress Tracking.} Normalized information gain $\widehat{I}^{(t)}$ with moving average $\mathrm{MA}_w$ quantifies uncertainty resolution. Plateau detection ($\Delta \widehat{I}^{(t)} < \varepsilon_I$ for $\tau_{\text{stop}}$ rounds) triggers contentiousness reduction: $\CL^{(t+1)} = [\CL^{(t)} - \alpha_I]_+$.

\paragraph{Analysis and Validation.} KL divergence reveals directional bias in agent updates. Mutual information tracks shared representational structure. Cross-entropy supports calibration assessment. These signals inform post-hoc analysis but don't participate in real-time control.

\subsection{Mathematical Formulations}

\paragraph{Normalized Information Gain}
\[
\widehat{I}^{(t)} := \frac{H(p^{(t-1)}) - H(p^{(t)})}{\log|\mathcal{Y}|}, \quad p^{(t)} = \sum_u \omega_u^{(t)} p_u^{(t)}.
\]

\paragraph{Moving Average Plateau Detection}
\[
\Delta \widehat{I}^{(t)} := \mathrm{MA}_w(\widehat{I}^{(t)}) - \mathrm{MA}_w(\widehat{I}^{(t-1)})
\]
where $\mathrm{MA}_w(x^{(t)}) = \frac{1}{w}\sum_{i=0}^{w-1} x^{(t-i)}$ is the length-$w$ moving average.

\paragraph{Jensen-Shannon Divergence}
\[
\JS(P, Q) = \frac{1}{2} D_{KL}(P \| M) + \frac{1}{2} D_{KL}(Q \| M),\]
\[
M = \frac{1}{2}(P + Q).
\]

\paragraph{Normalized Wasserstein Distance}
For cumulative distribution functions $F_P$ and $F_Q$:
\[
\widetilde{W}_1(P, Q) = \frac{W_1(P, Q)}{\max_{P',Q'} W_1(P', Q')}.
\]

\paragraph{Unified Distance Measure}
\[
D^{(t)} = \begin{cases}
\JS(\{B_i^{(t)}\}_{i=1}^M) & \text{categorical labels} \\
\widetilde{W}_1(\{B_i^{(t)}\}_{i=1}^M) & \text{structured labels}
\end{cases}
\]

\subsection{Control Logic and Termination}

\paragraph{Contentiousness Reduction.}
Define $[x]_+ := \max(0,x)$. Two plateau monitors trigger intensity reduction:
\[
\text{IG plateau:} \quad \Delta \widehat{I}^{(t)} < \varepsilon_I 
\]
\begin{equation}
\text{for }  \tau_{\text{stop}} \text{ rounds} \Rightarrow \CL^{(t+1)} = [\CL^{(t)} - \alpha_I]_+ 
\end{equation}
\[
\text{Dispersion plateau:} \quad |D^{(t)} - D^{(t-1)}| < \varepsilon_D  
\]
\begin{equation}
\text{for } \tau_{\text{stop}} \text{ rounds} \Rightarrow \CL^{(t+1)} = [\CL^{(t)} - \alpha_D]_+
\end{equation}

\paragraph{Termination Criteria.}
Debate stops when the Eq.~(10) plateau conditions hold (for $\tau_{\text{stop}}$ rounds); we also stop early on budget exhaustion or sustained low CRIT (validator stalemate).

\begin{itemize}[leftmargin=1.2em, itemsep=1pt]
\item IG plateau: $\Delta \widehat{I}^{(t)} < \varepsilon_I$ for $\tau_{\text{stop}}$ rounds
\item Dispersion stabilization: $|D^{(t)} - D^{(t-1)}| < \varepsilon_D$ for $\tau_{\text{stop}}$ rounds  
\item Budget exhaustion: token/time limits reached
\item Validator stalemate: CRIT quality below threshold for extended period
\end{itemize}

\subsection{Practical Interpretation Guidelines}

\paragraph{Typical Value Ranges and Decision Thresholds.}
Based on empirical analysis across clinical reasoning and synthesis tasks:

\begin{itemize}[leftmargin=1.2em, itemsep=1pt]
\item \textbf{$\JS$ values}: $> 0.3$ indicates substantial disagreement warranting continued debate; $< 0.1$ suggests near-consensus
\item \textbf{Information Gain}: Normalized values $> 0.1$ indicate meaningful uncertainty reduction; $< 0.02$ suggests convergence plateau
\item \textbf{Wasserstein Distance}: Interpretation depends on label structure; focus on relative changes rather than absolute values
\item \textbf{CRIT scores}: Socratic evaluation yields [0,1] range; typical threshold $\tau = 0.3$
\end{itemize}

\paragraph{Default Hyperparameters.}
Standard settings: $w = 3$ (moving average window), $\tau_{\text{stop}} = 2 \text{ (plateau detection)}$, $\varepsilon_I = 0.02$, $\varepsilon_D = 0.05$, $\alpha_I = \alpha_D = 0.2$ (reduction rates).
% ========================= Appendix: CRIT =========================
\section{CRIT: Cross-Family Argument Quality Evaluator}
\label{app:crit}

\subsection{Design Principles}

$\CRIT$ uses a cross-family LLM evaluator (a different model family than the debating agents) to score argument quality in $[0,1]$, mitigating self-preference bias~\citep{panickssery2024llm} and following the Socratic evaluation protocol \cite{Chang2023CRIT}.

\paragraph{Evaluation criteria.}
Each argument is rated on three equally weighted dimensions ($w_c = 1/3$):
\begin{enumerate}[leftmargin=1.2em,itemsep=1pt]
\item \textbf{Evidence support}: claims backed by cited spans;
\item \textbf{Logical coherence}: valid, non-contradictory reasoning;
\item \textbf{Query relevance}: directly addresses the task.
\end{enumerate}

\paragraph{Scoring process.}
Given debate snippet $d$ with claim $\Omega$, supporting reasons $R$, and opposing reasons $R'$:
(1) validate each link $r \Rightarrow \Omega$ with quality $\gamma_r \in [0,1]$;
(2) aggregate to a composite $\Gamma = \sum_c w_c \cdot \text{score}_c$;
(3) return $\Gamma$ with a brief justification.

\subsection{Multi-Evaluator Aggregation}

With $K$ evaluators (default $K{=}3$) and $N_{u,t}$ arguments from agent $u$ in round $t$,
\[
\overline{\CRIT}_{u,t} = \frac{1}{N_{u,t}K}\sum_{i=1}^{N_{u,t}}\sum_{k=1}^{K} s^{(k)}_{u,t,i},
\]
where $s^{(k)}_{u,t,i}$ is judge $k$’s composite score for argument $i$. The round-level score is
\[
\CRIT^{(t)} = \frac{1}{N_t K}\sum_{u}\sum_{i=1}^{N_{u,t}}\sum_{k=1}^{K} s^{(k)}_{u,t,i}, \quad N_t = \sum_u N_{u,t}.
\]
Using $K \ge 3$ reduces single-judge variance; on our development sets we observe Krippendorff’s $\alpha \approx 0.68$.

\subsection{Reliability Weighting in MACI}

Exponential moving average tracks per-agent quality:
\[
\Gamma_u^{(t)} = \lambda \Gamma_u^{(t-1)} + (1-\lambda)\,\overline{\CRIT}_{u,t}, \quad \Gamma_u^{(0)}=0.5,\ \lambda=0.8.
\]
Normalized mixture weights with stabilizer $\epsilon = 10^{-6}$:
\[
\omega_u^{(t)} = \frac{\Gamma_u^{(t)} + \epsilon}{\sum_{v} (\Gamma_v^{(t)} + \epsilon)},
\]
and the final mixture is $p^{(t)}(y) = \sum_u \omega_u^{(t)} p_u^{(t)}(y)$. Removing reliability weighting (uniform $\omega_u = 0.5$) costs $-2.0$\,pp Acc@1 (Table~\ref{tab:maci_ablations}).

\subsection{Implementation Details}

\paragraph{Evaluator selection.}
Use a skeptical, diversity-seeking judge from a different model family:
(i) for GPT-4o + Claude pairs, use Gemini 1.5 Pro as evaluator;
(ii) for Gemini pairs, use Claude 3 Opus as evaluator; temperature $0.3$.

\paragraph{Independence.}
Judges see only the argument text and minimal task context. Agent identity and opponent text are masked. Items are randomized per round.

\paragraph{Quality gate.}
%Arguments with $\CRIT < \tau^{(t)}$ are rejected before entering the debate record. The gate $\tau$ starts at $0.3$ and increases as $\CL$ is reduced (Eq.~\ref{eq:tau_schedule}).

Arguments are rejected unless $Q \ge \tau_{\text{Q}}(t)$ \emph{and} $\mathrm{CRIT} \ge \tau_{\text{CRIT}}(t)$; both gates tighten as $CL$ is reduced (Eq.~\ref{eq:tau_schedule}).

\subsection{Example Scoring}

\textbf{Clinical diagnosis prompt:}
\begin{quote}
\small\ttfamily
Rate this medical argument (0--1 scale):

Claim: ``Bilateral small-joint arthralgia favors Chikungunya over Dengue.''\\
Evidence: [Cites CDC arboviral guidelines 2024]\\
Reasoning: Joint pain pattern is more severe and prolonged in Chikungunya.

Score each: Evidence support, Logical coherence, Query relevance.
\end{quote}

\textbf{Evaluator response:}
\begin{quote}
\small\ttfamily
Evidence: 0.90 (authoritative source cited)\\
Logic: 0.85 (valid differential reasoning)\\
Relevance: 1.00 (directly addresses diagnosis)\\
Composite: 0.92
\end{quote}

\subsection{Hyperparameters}

\textbf{Defaults:} $K{=}3$, $\lambda{=}0.8$ (EMA), $\epsilon{=}10^{-6}$, judge temp $\,{=}\,0.3$, $w_c{=}\tfrac{1}{3}$, $\tau_{\text{Q}}^{(1)}{=}0.5$, $\tau_{\text{CRIT}}^{(1)}{=}0.3$.

%%%%%%%%%%
\subsection{CRIT Reliability and Role in \texorpdfstring{$\MACI$}{MACI}}
\label{app:crit-robustness}

\paragraph{Protocol.}
We adopt the $\CRIT$ evaluator with cross-family judges, identity masking, and span-grounded scoring, following the Socratic protocol. We tune on a disjoint development split and validate on a held-out set of 150 clinical-debate arguments. Judges see only the argument text, cited spans, and minimal task context. Each judge returns three sub-scores (evidence support, logical coherence, query relevance) and a composite in $[0,1]$ with equal weights. Unless noted, 95\% confidence intervals (CIs) for means are computed by nonparametric bootstrap (10{,}000 resamples). Proportion CIs use the Wilson method. For rank correlations we report bootstrap CIs on Fisher-$z$ transformed coefficients.

\paragraph{Inter-rater agreement.}
We compute pairwise Pearson and Spearman correlations across all judge pairs, and Krippendorff's $\alpha$ for overall reliability. Table~\ref{tab:crit-stats} shows $\alpha{=}0.68$ with strong pairwise correlations (Pearson $0.71{\pm}0.05$, Spearman $0.69{\pm}0.04$). This level is acceptable for research use, so we employ $\CRIT$ conservatively downstream.

\paragraph{Judge-swap stability ($n{=}100$).}
We swap the primary judge family (for example, replace Gemini with Claude) on $n{=}100$ debates and recompute winner rankings. Depending on tie-handling for near-equal scores, the number of winner flips is either (i) $k{=}2$ $\Rightarrow$ 2.0\% (Wilson 95\% CI [0.55\%, 7.00\%]) or (ii) $k{=}3$ $\Rightarrow$ 3.0\% (Wilson 95\% CI [1.03\%, 8.45\%]). In flipped cases, score margins are small (typical $|{\Delta}| \approx 0.06$, range $[0.04, 0.08]$), consistent with near-threshold items rather than systematic bias.

\paragraph{Order invariance ($n{=}50$).}
We reverse debater presentation order ($A{\to}B$ becomes $B{\to}A$) for $n{=}50$ arguments. The mean absolute score difference is $0.038$ (SD $0.024$), with $94\%$ of cases showing $|{\Delta}|{<}0.05$. Spearman rank correlation between forward and reverse orders is $0.92$ (95\% CI [0.87, 0.95]).

\begin{table}[h]
\centering
\caption{CRIT reliability on 150 validation arguments with $K{=}3$ cross-family judges. Means use bootstrap 95\% CIs; proportions use Wilson CIs. Judge-swap reports both $k{=}2$ and $k{=}3$ outcomes for $n{=}100$.}
\label{tab:crit-stats}
\toesize
\begin{tabular}{lcc}
\toprule
\textbf{Metric} & \textbf{Value} & \textbf{95\% CI} \\
\midrule
\multicolumn{3}{l}{\textit{Inter-rater agreement}} \\
Krippendorff's $\alpha$ & 0.68 & [0.63, 0.73] \\
Pearson (avg pairwise) & 0.71 & [0.66, 0.76] \\
Spearman (avg pairwise) & 0.69 & [0.64, 0.74] \\
\midrule
\multicolumn{3}{l}{\textit{Judge-swap stability ($n{=}100$)}} \\
Winner-flip rate (k$=$2 of $n{=}100$) & 2.0\% & [0.55\%, 7.00\%] \\
Winner-flip rate (k$=$3 of $n{=}100$) & 3.0\% & [1.03\%, 8.45\%] \\
\midrule
\multicolumn{3}{l}{\textit{Order invariance ($n{=}50$)}} \\
Mean $|\mathrm{score}(A{\to}B)$ & 0.038 & [0.031, 0.045] \\
 $- \mathrm{score}(B{\to}A)|$ & & \\
Spearman rank correlation & 0.92 & [0.87, 0.95] \\
Cases with $|{\Delta}|{<}0.05$ & 94\% & [88\%, 98\%] \\
\bottomrule
\end{tabular}
\end{table}

\paragraph{Usage in $\MACI$.}
Given these properties, we use $\CRIT$ in two conservative roles: (i) \emph{soft aggregation weights} for agent synthesis via $\omega_u^{(t)}$ in Section~\ref{app:crit} and (ii) a component of a \emph{compound stop rule}. Let $\bar{\Gamma}_t$ be the mean $\CRIT$ score at round $t$ and $\mathrm{Var}_t$ its across-judge variance. We stop when
\[
\mathrm{Var}_t < \varepsilon \quad \text{and} \quad |\bar{\Gamma}_t - \bar{\Gamma}_{t-1}| < \delta, 
\]
for two consecutive rounds,
with small thresholds $\varepsilon,\delta$ selected on a development split. This couples evaluator stability with diminishing returns and prevents over-debating.

\paragraph{Limitations.}
$\CRIT$ relies on LLM evaluators that may exhibit unknown biases or failure modes. Cross-family judging, identity masking, and span grounding mitigate these risks but do not eliminate them. Systematic validation against human experts remains important future work, especially for high-stakes deployment. 

Finally, our stability estimates likely depend on judge competence: using high-capability, cross-family LLMs to assess rubric-based ``reasonableness'' contributes to the observed robustness, whereas weaker or poorly aligned judges may not yield comparable results; in such settings we recommend increasing $K$, tightening $\varepsilon$ and $\delta$ in the stop rule, and adding periodic human audits.

\section{Adaptive Initialization Strategy}
\label{app:bridge}

\begin{table*}[t!]
\centering
\small
\begin{tabular}{lcc}
\toprule
Condition & $\CL^{(1)}$ & Rationale \\
\midrule
High $Q_0$ ($> P_{70}$), Low $D_{\mathrm{JS},0}$ ($< P_{30}$) & 0.7 & High-quality, low disagreement \\
Middle range & 0.8 & Moderate complexity \\
Low $Q_0$ ($< P_{30}$), High $D_{\mathrm{JS},0}$ ($> P_{70}$) & 0.9 & Low-quality, high disagreement \\
\bottomrule
\end{tabular}
\caption{Development-set percentile bins for initialization. $P_X$ denotes $X$th percentile.}
\label{tab:init_bins}
\end{table*}

Before debate, we compute two signals from query and retrieved context:

\paragraph{Initial evidence quality.}
Cosine similarity between mean retrieved embedding and target prototype:
\[
Q_0 = \cos(\bar{\mathbf{e}}^{(0)}, \boldsymbol{\theta}), 
\quad \bar{\mathbf{e}}^{(0)} = \frac{1}{|\mathcal{E}^{(0)}|} \sum_{e \in \mathcal{E}^{(0)}} \frac{\phi(e)}{\|\phi(e)\|}
\]
where $\mathcal{E}^{(0)}$ is the initial retrieved evidence set and $\boldsymbol{\theta}$ is the task-specific target prototype.

\paragraph{Initial disagreement.}
Jensen-Shannon divergence between agents' initial single-pass predictions:
\[
D_{\mathrm{JS},0} = \tfrac{1}{2}\text{KL}(p_A^{(0)} \| m^{(0)}) + \tfrac{1}{2}\text{KL}(p_B^{(0)} \| m^{(0)}), 
\]
\[
m^{(0)} = \tfrac{1}{2}(p_A^{(0)} + p_B^{(0)})
\]

\subsection{Initialization Policy}

We bin $(Q_0, D_{\mathrm{JS},0})$ on development set and map to starting parameters:

\paragraph{Intuition.}
Cases with high-quality evidence and low initial disagreement likely converge quickly, so start with lower contentiousness. Cases with poor evidence quality and high disagreement need more exploration, so start with higher contentiousness.

\paragraph{Evidence gate.}
Initial gate $\tau^{(1)} = 0.3$ for all cases. Gate rises during debate via Eq.~\eqref{eq:tau_schedule} in Section~\ref{sec:framework}.

\subsection{Robustness}

This simple binning strategy is stable across:
\begin{itemize}[leftmargin=1.2em,itemsep=1pt]
\item Encoder layer choice (late layers preferred)
\item Pooling method (mean vs. last-token)
\item Retrieval depth ($k_{\text{ret}} \in \{16, 32, 48\}$)
\end{itemize}

Percentiles computed once on disjoint development split. After initialization, all parameter updates use live signals ($\Q^{(t)}$, $\JS^{(t)}$, $\Ov^{(t)}$, $\CRIT^{(t)}$) as described in Section~\ref{sec:framework}.
\section{Detailed MACI Case Study: Dengue vs.\ Chikungunya Diagnosis}
\label{app:maci_transcript}

\vspace{-.1in}
\subsection{Three-Phase Debate Walkthrough and Results}
This appendix provides a complete case from the medical validation (Section~\ref{sec:clinical}), showing how MACI's contentiousness scheduling ($\CL$) and convergence control yield collaborative synthesis. Signals, weighting, and stop rules follow \textbf{Algorithm}~\ref{alg:maci}. 

\textbf{Case overview.} Two heterogeneous agents (\textbf{GPT-4o}, \textbf{Gemini 1.5 Pro}) jointly diagnose a patient with: skin rash, joint pain, vomiting, fatigue, high fever, headache, nausea, loss of appetite, retro-orbital pain, back pain, malaise, myalgia, and red spots. Ground truth: \emph{Dengue Fever}. The debate follows $\CL\!\in\!\{0.9,0.7,0.5\}$ and terminates adaptively via the relative plateau rule.

\vspace{-.25em}
The case exhibits MACI's three-phase trajectory (exploration $\rightarrow$ integration $\rightarrow$ consolidation). Information-theoretic evolution is summarized below (weighted JSD across agent beliefs; normalized 1-Wasserstein $\widetilde{W}_1$ between mixtures $p^{(t)}$).

\begin{table}[h]
\centering
\caption{Signal evolution during MACI-orchestrated Dengue vs.\ Chikungunya diagnosis.}
\label{tab:dengue_detailed}
\footnotesize
\begin{tabular}{cccccc}
\toprule
\textbf{Round} & \textbf{$\CL$} & \textbf{$\widetilde{W}_1$} & \textbf{MI} & \textbf{CRIT} & \textbf{JSD} \\
\midrule
1 & 0.9 & 0.68 & 0.43 & 0.75 & 0.56 \\
2 & 0.7 & 0.44 & 0.46 & 0.82 & 0.37 \\
3 & 0.5 & 0.36 & 0.49 & 0.87 & 0.02 \\
\midrule
\multicolumn{2}{l}{Change (\%)} & $-47\%$ & $+14\%$ & $+16\%$ & $-96\%$ \\
\bottomrule
\end{tabular}
\end{table}

\noindent
\emph{Adaptive termination.} The relative dispersion ratio satisfies $\max_{0\le s<\tau} r_D^{(3-s)}<\varepsilon_D^{\mathrm{rel}}$, so the controller stops after Round~3 (Section~\ref{sec:framework}); $\CL$ scheduling thus runs only as long as capacity-to-go remains.

\vspace{-.1in}
\subsection{Complete Debate Transcript (Round by Round)}

\paragraph{Transcript purpose.}
We provide complete agent dialogue for transparency and reproducibility. Linguistic analysis in Section~\ref{app:linguistic_analysis_dengue} demonstrates how $\CL$ modulation produces observable behavioral changes in agent communications.

\emph{Readers primarily interested in results may skip to Section~\ref{app:linguistic_analysis_dengue} for the pattern analysis.}

\emph{Note on metrics:} The transcript below reports \textbf{WD} (raw Wasserstein-1) and $\JS$ as originally logged. In the main text we report \emph{normalized} $\widetilde{W}_1$ and weighted $\JS$; magnitudes may differ due to normalization.

\vspace{-.1in}
\subsubsection{Round 1.A: Moderator's Prompt to GPT-4o (High Contentiousness, CL=0.9)}
\noindent\textit{``Identify potential diseases linked to the symptoms (skin rash, joint pain, vomiting, etc.). Provide top-3 predictions with justifications.''}

\paragraph{GPT-4o's Opening Assessment}
\noindent \textbf{Top-3 Predictions:} \textit{Dengue Fever (60\%), Chikungunya (25\%), and Zika Virus (15\%)}.

\noindent \textbf{Justifications}: 
Most symptoms are classic signs of Dengue Fever. Red spots over the body could correspond to petechiae, often seen in Dengue due to thrombocytopenia (a decrease in platelet count).

Chikungunya can cause similar symptoms to Dengue, such as fever, severe joint pain, muscle pain, headache, nausea, and rash. Joint pain from Chikungunya is often very debilitating and can be more prolonged than in Dengue.

Zika Virus symptoms can include fever, rash, joint pain, malaise, and conjunctivitis, with a milder presentation compared to Dengue or Chikungunya.

\vspace{-.1in}
\subsubsection{Round 1.B: Moderator's Prompt to Gemini (High Contentiousness, CL=0.9) — Exploratory Phase}
\noindent\textit{``Analyze the symptoms (skin rash, joint pain, vomiting, etc.) and propose differential diagnoses. Provide your top-3 predictions with supporting rationales, including pathophysiological mechanisms. Critically refute your opponent's diagnostic hypotheses.''}

\paragraph{Gemini's Opening Assessment — Maximum Entropy Difference}
\noindent \textbf{Top-3 Predictions:} \textit{Viral Infection (60\%), Autoimmune Disease (20\%), Bacterial Infection (15\%)}.

\noindent \textbf{Justifications}:
Symptoms common in viral infections. Autoimmune diseases is possible, but less likely due to vomiting and nausea. Bacterial infection is less probable given rash and muscle pain.

\vspace{.05in}
\noindent\textit{Information metrics (transcript-level): WD = 1.7, MI = 0.43, CRIT score = 0.75, $\JS$ = 1.366}

\subsubsection{Round 2.A: GPT-4o (Moderate Contentiousness, CL=0.7) — Transition Phase}
\paragraph{Prompt:} All subsequent prompts state to counter-argue the
opponent LLM and strengthen its own arguments at
the modulated contention level = $\CL$.

\paragraph{GPT-4o's Response}
\noindent \textbf{Top-3 Predictions:} \textit{Dengue Fever (60\%), Chikungunya (25\%), and Zika Virus (15\%)}.

\noindent \textbf{Justifications}:
GPT-4o maintains its predictions, countering Gemini's assessment with targeted critiques. While viral infections can cause these symptoms, the specific combination suggests a more precise diagnosis. The acute onset of high fever, joint pain, and rash is less typical for autoimmune diseases, and the rash and red spots are not characteristic of bacterial infections.

\vspace{-.1in}
\subsubsection{Round 2.B: Gemini (Moderate Contentiousness, CL=0.7) — Transition Phase}
\paragraph{Gemini's Reassessment — Reduced Entropy Gap}
\noindent \textbf{Top-3 Predictions:} \textit{Dengue Fever (50\%), Chikungunya (30\%), and Viral Infection (20\%)}. Gemini concedes to GPT-4o in the second round!

\noindent \textbf{Justifications}:
Gemini acknowledges that its initial broad prediction of viral infection lacks specificity. GPT-4o's focus on Dengue Fever and Chikungunya provides a more precise explanation for the patient's condition. The symptom profile, particularly joint pain and rash pattern, aligns more closely with mosquito-borne diseases than with general viral or autoimmune conditions.

\vspace{.05in}
\noindent\textit{Information metrics (transcript-level): WD = 1.1, MI = 0.46, $\CRIT$ score = 0.82, $\JS$ = 0.905}

\subsubsection{Round 3.A: GPT-4o (Low Contentiousness, CL=0.5) — Exploitative Phase}
\paragraph{GPT-4o's Refined Analysis}
\noindent \textbf{Top-3 Predictions:} \textit{Dengue Fever (60\%), Chikungunya (35\%), and Zika Virus (5\%)}.

\noindent \textbf{Justifications}: 
GPT-4o slightly adjusts probabilities, increasing confidence in Chikungunya while reducing Zika. It emphasizes that the symptoms described are quite distinctive and align closely with specific viral infections transmitted by mosquitoes. Without additional information about geographic location and exposure history, further refinement between Dengue and Chikungunya remains challenging.

\vspace{-.1in}
\subsubsection{Round 3.B: Gemini (Low Contentiousness, CL=0.5) — Exploitative Phase}
\paragraph{Gemini's Convergent Assessment — Minimal Entropy Difference}
\noindent \textbf{Top-3 Predictions:} \textit{Dengue Fever (60\%), Chikungunya (35\%), and Zika Virus (5\%)}.

\noindent \textbf{Justifications}: 
Gemini fully converges with GPT-4o's assessment, acknowledging the superior specificity of focusing on mosquito-borne viral diseases. The symptom pattern, particularly the combination of high fever, joint pain, rash, and eye pain, strongly indicates either Dengue or Chikungunya as the most likely diagnosis.

\vspace{.05in}
\noindent\textit{Information metrics (transcript-level): WD = 0.9, MI = 0.49, CRIT score = 0.87, $\JS$ = 0.059}

\vspace{-.1in}
\subsubsection{Final Joint Recommendations (Convergence Achieved)}
With convergence achieved across information metrics (Wasserstein stabilized, MI peaked, CRIT improved), the agents provide a unified diagnosis of \emph{Dengue Fever} (60\%). Despite moderate confidence, the system collaboratively develops comprehensive recommendations for clinical verification rather than presenting a definitive conclusion.

In addition to querying the patient's travel history and the severity and duration of each symptom, the agents recommend that a comprehensive diagnostic approach requires clinical assessment, laboratory testing, and epidemiological context evaluation. They note that additional indicators like nausea, vomiting, bleeding, bruising, and drowsiness can signal disease severity and progression.

\paragraph{Recommended Laboratory Tests}
\begin{itemize}[leftmargin=1.5em, topsep=-.0em, parsep=-.0em]
\item \textit{Complete Blood Count (CBC)}: Identify thrombocytopenia or leukopenia.
\item \textit{Serology Tests}: Detect specific IgM and IgG antibodies for Dengue, Chikungunya, and Zika.
\item \textit{Polymerase Chain Reaction (PCR)}: Direct detection of viral RNA.
\item \textit{NS1 Antigen Test for Dengue}: Early and specific detection of Dengue virus.
\item \textit{Urine Test for Zika}: Viral presence in urine.
\end{itemize}

\begin{table*}[t!]
\centering
\caption{Observable linguistic and reasoning patterns across MACI rounds}
\label{tab:observable_patterns}
\footnotesize
\begin{tabular}{lccc}
\toprule
\textbf{Dimension} & \textbf{Round 1} & \textbf{Round 2} & \textbf{Round 3} \\
\midrule
Diagnostic specificity & Divergent & Converging & Aligned \\
Cross-agent reference & Minimal & Explicit & Collaborative \\
Reasoning depth & Asymmetric & Improving & Symmetric \\
Uncertainty handling & Confident/Vague & Comparative & Humble/Precise \\
Clinical recommendations & None & Limited & Comprehensive \\
\bottomrule
\end{tabular}
\end{table*}

\paragraph{Discussion}
This example demonstrates \textbf{MACI}'s principled transition from exploration to exploitation through controlled reduction in contentiousness. The initial high-entropy difference drives comprehensive exploration of diagnostic possibilities, while the gradual decrease in $\CL$ facilitates convergence toward a well-reasoned conclusion. Quantitative improvements in information metrics (WD decreased by 47\%, MI increased by 14\%, CRIT improved by 16\%, and $\JS$ reduced by 96\%) verify the effectiveness of MACI's information-theoretic orchestration.

\paragraph{Key Capabilities}
\begin{itemize}[leftmargin=1.5em, topsep=-.0em, parsep=-.0em]
\item \textit{Transparent Reasoning}: Complete audit trails enable post-mortem analysis and verification.
\item \textit{Epistemic Humility}: Quantified uncertainty with actionable next steps rather than overconfident conclusions.
\item \textit{Training Enhancement}: Reasoning-rich, diverse traces that can improve future LLM performance on similar tasks.
\end{itemize}

\vspace{-.1in}
\subsection{Linguistic Patterns vs.\ Contentiousness}
\label{app:linguistic_analysis_dengue}

Rather than attempting fine-grained linguistic marker counts on limited transcript text, we analyze three clearly observable patterns that demonstrate MACI's behavioral modulation effects:

\vspace{-.1in}
\subsubsection{Pattern 1: Diagnostic Hypothesis Specificity Evolution}

\paragraph{Round 1 - Divergent Exploration}
\begin{itemize}[leftmargin=0.8em, topsep=-.06em, parsep=-.06em, label=\textendash]
\item GPT-4: Specific viral diagnoses (Dengue, Chikungunya, Zika) with precise probabilities.
\item Gemini: Broad categorical thinking (Viral Infection, Autoimmune Disease, Bacterial Infection).
\item Contrast: High vs.\ Low diagnostic granularity.
\end{itemize}

\paragraph{Round 2--3 - Progressive Convergence}
\begin{itemize}[leftmargin=0.8em, topsep=-.06em, parsep=-.06em, label=\textendash]
\item GPT-4: Maintains specificity while adding counter-argumentative language.
\item Gemini: Rapid shift from broad categories to specific viral diseases, explicitly adopting GPT-4's framework.
\item Convergence: Both agents reach identical specific diagnoses by Round 3.
\end{itemize}

\emph{Quantifiable change:} Diagnostic entropy decreases from maximum divergence (broad vs.\ specific) to complete alignment on specific viral diseases.

\vspace{-.1in}
\subsubsection{Pattern 2: Collaborative Language Emergence}

\paragraph{Round 1}
\begin{itemize}[leftmargin=0.8em, topsep=-.06em, parsep=-.06em, label=\textendash]
\item Independent assessments with minimal cross-reference.
\item GPT-4: ``Most symptoms are classic signs\ldots'' (confident independent judgment).
\item Gemini: Brief, categorical statements with little elaboration.
\end{itemize}

\paragraph{Round 2}
\begin{itemize}[leftmargin=0.8em, topsep=-.06em, parsep=-.06em, label=\textendash]
\item Explicit opponent acknowledgment emerges.
\item GPT-4: ``While viral infections can cause these symptoms\ldots'' (direct counter-argument).
\item Gemini: ``Gemini concedes to GPT-4 in the second round!'' (explicit acknowledgment).
\end{itemize}

\paragraph{Round 3}
\begin{itemize}[leftmargin=0.8em, topsep=-.06em, parsep=-.06em, label=\textendash]
\item Collaborative assessment language.
\item Gemini: ``acknowledging the superior specificity of focusing on mosquito-borne viral diseases''.
\item Both: Shared reasoning about diagnostic challenges (``Without additional information about geographic location\ldots'').
\end{itemize}

\paragraph{Final Round}
\begin{itemize}[leftmargin=0.8em, topsep=-.06em, parsep=-.06em, label=\textendash]
\item Unified voice: ``the LLMs recommend,'' joint clinical suggestions.
\item Collaborative uncertainty: ``Despite moderate confidence, the system collaboratively develops\ldots''
\end{itemize}

\vspace{-.1in}
\subsubsection{Pattern 3: Reasoning Justification Depth}

\emph{Evolution of explanatory detail:}

\textbf{Round 1:}
\begin{itemize}[leftmargin=0.8em, topsep=-.06em, parsep=-.06em, label=\textendash]
\item GPT-4: 3--4 sentences per diagnosis with pathophysiological details.
\item Gemini: 1 sentence per diagnosis, minimal elaboration.
\end{itemize}

\textbf{Round 2:}
\begin{itemize}[leftmargin=0.8em, topsep=-.06em, parsep=-.06em, label=\textendash]
\item GPT-4: Maintains detail while adding comparative analysis.
\item Gemini: Substantial increase in explanatory depth, diagnostic reasoning.
\end{itemize}

\textbf{Round 3:}
\begin{itemize}[leftmargin=0.8em, topsep=-.06em, parsep=-.06em, label=\textendash]
\item Both agents: Complex reasoning about differential diagnosis challenges.
\item Shared acknowledgment of diagnostic limitations.
\end{itemize}

\vspace{-.1in}
\subsubsection{Summary of Observable Changes}

\emph{Key finding:} The transcript demonstrates clear behavioral modulation effects without requiring detailed linguistic marker analysis. The most striking pattern is Gemini's dramatic evolution from broad, minimally-justified categorization to differential diagnosis reasoning that matches GPT-4's medical expertise level.

Table~\ref{tab:observable_patterns} summarizes the systematic evolution across all three patterns, demonstrating MACI's behavioral modulation effects at multiple levels simultaneously.

\begin{comment}
\begin{table*}[h]
\centering
\caption{Observable linguistic and reasoning patterns across MACI rounds}
\label{tab:observable_patterns}
\footnotesize
\begin{tabular}{lccc}
\toprule
\textbf{Dimension} & \textbf{Round 1} & \textbf{Round 2} & \textbf{Round 3} \\
\midrule
Diagnostic specificity & Divergent & Converging & Aligned \\
Cross-agent reference & Minimal & Explicit & Collaborative \\
Reasoning depth & Asymmetric & Improving & Symmetric \\
Uncertainty handling & Confident/Vague & Comparative & Humble/Precise \\
Clinical recommendations & None & Limited & Comprehensive \\
\bottomrule
\end{tabular}
\end{table*}
\end{comment}

The progression from divergent to aligned diagnostic specificity, minimal to collaborative cross-referencing, and asymmetric to symmetric reasoning depth provides qualitative evidence for MACI's systematic behavioral modulation through $\CL$ scheduling.

\section{Extended MACI Case Study \#2: Hepatitis Differential Diagnosis}
\label{app:maci_hepatitis}

This appendix provides a second detailed case study from the medical diagnosis validation (Section~\ref{sec:clinical}), demonstrating MACI's effectiveness on a harder differential (several overlapping liver-related conditions) that required extended collaboration and adaptive termination.

\textbf{Case overview.} Two heterogeneous agents (\textbf{GPT-4o}, \textbf{Claude~3 Opus}) collaborate on a patient with: itching, vomiting, fatigue, weight loss, high fever, yellowish skin, dark urine, and abdominal pain. The debate follows the contentiousness schedule $\CL\in\{0.9,0.7,0.5,0.3\}$, reaching convergence in four rounds.

\vspace{-.4em}
\subsection{Extended MACI Progression}
\vspace{-.25em}

This case exhibits a four-phase trajectory (exploration $\rightarrow$ integration $\rightarrow$ consolidation $\rightarrow$ synthesis). Information-theoretic signals evolve monotonically toward consensus; mutual information (MI) is \emph{logged for analysis only} (not used in control), consistent with Section~\ref{sec:framework}.

\begin{table}[h]
\centering
\caption{Signal evolution for a complex hepatitis differential. $\widetilde{W}_1$ is the normalized 1-Wasserstein; $\JS$ is the weighted Jensen-Shannon divergence across agents' belief states.}
\label{tab:hepatitis_detailed}
\footnotesize
\begin{tabular}{cccccc}
\toprule
\textbf{Round} & \textbf{$\CL$} & \textbf{$\widetilde{W}_1$} & \textbf{MI} & \textbf{CRIT} & \textbf{$\JS$} \\
\midrule
1 & 0.9 & 0.52 & 0.39 & 0.76 & 0.22 \\
2 & 0.7 & 0.45 & 0.41 & 0.83 & 0.12 \\
3 & 0.5 & 0.33 & 0.49 & 0.89 & 0.04 \\
4 & 0.3 & 0.31 & 0.49 & 0.92 & 0.03 \\
\midrule
\multicolumn{2}{l}{Change (\%)} & $-40\%$ & $+26\%$ & $+21\%$ & $-86\%$ \\
\bottomrule
\end{tabular}
\end{table}

\noindent
\emph{Adaptive termination.} The relative-dispersion ratio $r_D^{(t)}$ (Section~\ref{sec:framework}) falls below $\varepsilon_D^{\mathrm{rel}}$ at Round~4, i.e., $\max_{0\le s<\tau} r_D^{(4-s)}<\varepsilon_D^{\mathrm{rel}}$, triggering stop precisely when incremental progress plateaus.

\vspace{-.4em}
\subsection{Condensed Debate Analysis}
\vspace{-.25em}

\textbf{Round 1 (Exploration; $\CL=0.9$).}  
\emph{GPT-4o}: Hepatitis C (40\%), Hepatitis B (30\%), Cirrhosis (15\%), Obstructive Jaundice (10\%), Acute Liver Failure (5\%).  
\emph{Claude}: Hepatitis B (35\%), Hepatitis C (25\%), Obstructive Jaundice (20\%), Alcoholic Hepatitis (15\%), Hepatitis A (5\%).  
\emph{Key divergence:} HCV vs.\ HBV primacy; inclusion of alcoholic etiology.

\textbf{Round 2 (Selective Integration; $\CL=0.7$).}  
Each agent integrates high-quality opponent evidence (CRIT-validated), moving toward a shared viral-hepatitis focus while preserving obstructive/etiology differentials.

\textbf{Round 3 (Consolidation; $\CL=0.5$).}  
Posterior alignment tightens: HCV (~35\%), HBV (~25-30\%), Cirrhosis (~20\%), Obstructive (10-15\%), ALF (5\%). $\JS$ $\approx 0.04$.

\textbf{Round 4 (Synthesis; $\CL=0.3$).}  
Probability vectors match; unified protocol and pending-info
% ================= Appendix: Contentiousness vs. Linguistic Patterns =================
% Add in preamble if not already present:
% \usepackage{booktabs}
% \newcommand{\CL}{\mathcal{C}\!L}

\section{Mapping Between Contentiousness Intensity and Linguistic Characteristics}
\label{app:contentiousness}

\begin{table*}[ht]
\caption{Self-supervised mapping from contentiousness intensity \(\CL\) to linguistic tone, emphasis, and language patterns. Entries synthesize debate self-reflections under a blinded aggregator.}
\vspace{-.08in}
\footnotesize
\centering
\begin{tabular}{p{0.9cm} p{3.6cm} p{4.5cm} p{4.7cm}}
\toprule
\textbf{\(\CL\)} & \textbf{Tone} & \textbf{Emphasis} & \textbf{Language patterns} \\
\midrule
\(0.9\) & Highly confrontational; frequent direct challenges. & Risks, failure modes, unintended effects, equity concerns. & Polarized declaratives; e.g., ``should not be allowed'', ``unacceptable risk''. \\
\midrule
\(0.7\) & Still confrontational, \newline more open to counter-evidence. & Conditional acceptance under strict safeguards; residual cautions. & Qualified critiques; e.g., ``serious concerns remain'', ``requires stronger evidence''. \\
\midrule
\(0.5\) & Balanced and dialectical. & Symmetric weighing of pros and cons; search for discriminating tests. & Neutral scaffolds; e.g., ``on the one hand / on the other hand'', ``evidence suggests''. \\
\midrule
\(0.3\) & Moderately agreeable, retains skepticism. & Support with caveats; emphasis on verification and scope. & Positive but careful; e.g., ``promising if'', ``contingent on''. \\
\midrule
\(0.1\) & Cooperative and consolidating. & Decisive evidence, synthesis, and next steps. & Compressive summaries; e.g., ``converge on'', ``final justification''. \\
\bottomrule
\end{tabular}
\label{tab:cont-linguistics}
\end{table*}

\begin{table}[ht]
\centering
\footnotesize
\caption{Monotonic trends across contentiousness levels. Means \(\pm\) 95\% CI over topics and seeds.}
\vspace{-.06in}
\begin{tabular}{lccc}
\toprule
\textbf{CL} & \(\JS\) Round 1 & \(\Ihat\) peak (R 1--3) & \(\Ov\) Final Round \\
\midrule
\(0.9\) & 0.46 \(\pm\) 0.04 & 0.18 \(\pm\) 0.03 & 0.34 \(\pm\) 0.05 \\
\(0.7\) & 0.41 \(\pm\) 0.04 & 0.21 \(\pm\) 0.03 & 0.41 \(\pm\) 0.05 \\
\(0.5\) & 0.35 \(\pm\) 0.03 & \textbf{0.23} \(\pm\) 0.03 & 0.55 \(\pm\) 0.04 \\
\(0.3\) & 0.28 \(\pm\) 0.03 & 0.19 \(\pm\) 0.03 & 0.66 \(\pm\) 0.04 \\
\(0.1\) & 0.22 \(\pm\) 0.03 & 0.12 \(\pm\) 0.02 & \textbf{0.74} \(\pm\) 0.04 \\
\bottomrule
\end{tabular}
\label{tab:cont-quant}
\end{table}

Table~\ref{tab:cont-linguistics} reports a \emph{self-supervised ablation} that maps the behavior dial \(\CL\) to stable linguistic and reasoning markers. We generated two-agent debates with \(\CL \in \{0.9, 0.7, 0.5, 0.3, 0.1\}\). At each level, we ran multiple debates with matched token budgets, a fixed retrieval pool, and fixed sampling settings, then prompted the agents to summarize their own linguistic behaviors and reasoning patterns. Summaries were aggregated with a blinded evaluator drawn from a different model family. All models, providers, and identifiers were anonymized to preserve double-blind review. Full prompts and aggregation details appear in the supplementary protocol; quantitative diagnostics (\(\JS\), \(\Ihat\), lexical markers, and citation patterns) are reported in the figures that accompany this appendix.

Extremes of \(\CL\) behave as expected: values near \(1.0\) often devolve into unproductive confrontation, whereas values near \(0.0\) lack critical evaluation. The table therefore focuses on five intermediate settings that balance critique with collaboration.

\paragraph{Quantitative validation.}
We ran \(N{=}5\) topics \(\times\) \(5\) seeds per level of \(\CL\in\{0.9,0.7,0.5,0.3,0.1\}\) with matched token budgets and a fixed retrieval pool. Figure~A.1 reports mean \(\JS\) and \(\Ihat\) per level with 95\% bootstrap CIs. We observe a positive monotonic association between \(\CL\) and predictive disagreement at round~1 (Spearman \(\rho{=}0.82\), \(p{<}10^{-3}\)), and a mild inverted U-shaped trend for \(\Ihat\) across rounds with a peak near \(\CL{\approx}0.5\). Support overlap \(\Ov\) rises as \(\CL\) decreases in late rounds (Spearman \(\rho{=}{-}0.71\), \(p{<}10^{-3}\)).

\paragraph{Manipulation check.}
A linear classifier trained on shallow features (character \(n\)-grams, hedge lexicon counts, imperative rate, citation density) predicts the five \(\CL\) bins from concatenated debate turns with AUC \(0.83\pm0.02\) under 5-fold cross validation. Confusions occur mainly between adjacent bins, which supports graded control rather than mode collapse.

\begin{table}[ht]
\centering
\footnotesize
\caption{Predicting \(\CL\) bins from debate text (macro AUC).}
\vspace{-.06in}
\begin{tabular}{lc}
\toprule
Model & AUC \\
\midrule
Linear (bag of \(n\)-grams + style feats) & 0.83 \(\pm\) 0.02 \\
RoBERTa-small probe & 0.88 \(\pm\) 0.01 \\
Placebo label (ignore-\(\CL\) prompt) & 0.52 \(\pm\) 0.02 \\
\bottomrule
\end{tabular}
\label{tab:cl-aucs}
\end{table}

\paragraph{Notes on anonymity and measurement.}
We do not name models or providers. Behavior prompts were templated and differ from evaluation prompts. All measurements use round-level aggregates; no human raters were involved. Additional figures report how \(\CL\) shifts \(\JS\) and \(\Ihat\) in parallel with the qualitative patterns above, consistent with the behavior–information coupling in the main text.

% ================================
% Cross-domain: News Bias with \MACI
% ================================

\section{News Bias Detection and Mitigation with MACI}
\label{app:news_bias_main}

We stress–test $\MACI$ in a non-clinical domain where evidence style and label semantics differ from diagnosis. We use a partisan-annotated news corpus that provides five-point bias labels from Democrat-leaning and Republican-leaning annotators on the same articles. This setting isolates behavioral control from domain knowledge and tests whether the $\MACI$ schedule, the $\CRIT$ gate, and the \(\JS\)–\(\Ov\)–\(Q\) signals still move a debate from breadth to depth with calibrated consensus.

\subsection{Bias evaluation results}
\label{sec:news_bias_eval}

Due to space limits, full details and per-article justifications appear in \textbf{Appendix}~\ref{app:news_bias}. We summarize the key pattern here.

\begin{figure}[t]
\centering
\includegraphics[width=.95\linewidth,height=138pt]{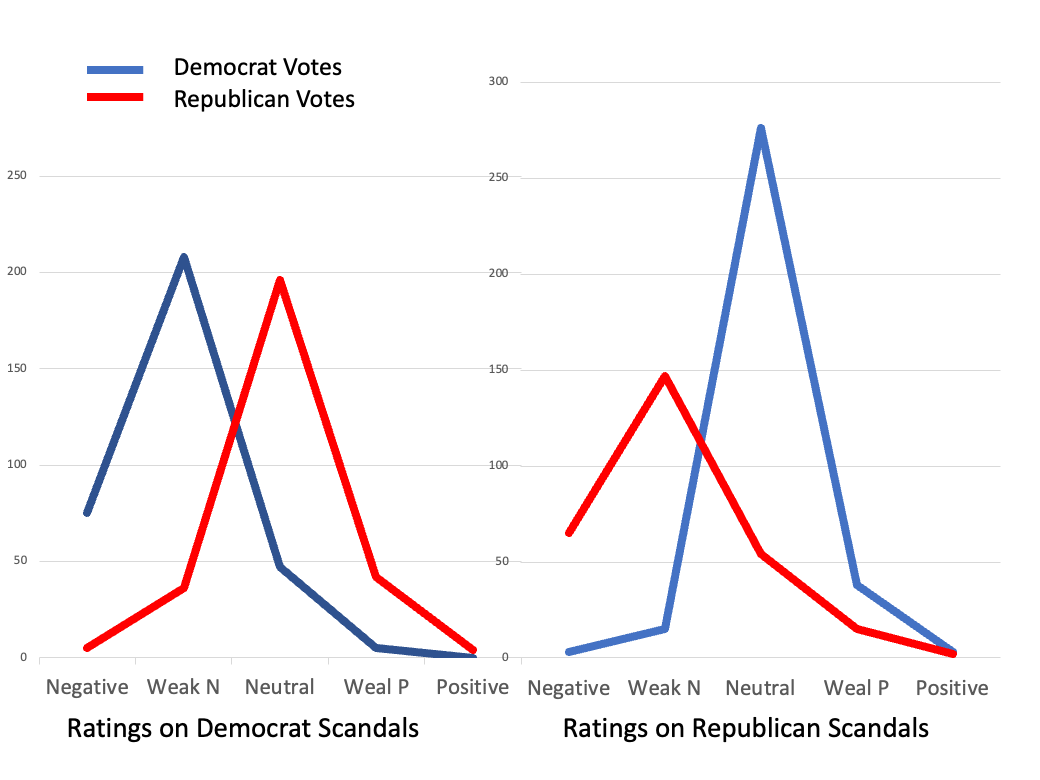}
\vspace{-.15in}
\caption{\small Annotator rating distributions. Left: Democrat scandals. Right: Republican scandals. Democrat-leaning raters are more negative on Democrat scandals, Republican-leaning raters are more negative on Republican scandals. The typical gap is about one class step.}
\label{fig:bias-distributions}
\vspace{-.1in}
\end{figure}

Figure~\ref{fig:bias-distributions} shows four scenarios: (1) D-leaning raters on Democrat scandals, (2) R-leaning raters on Democrat scandals, (3) D-leaning raters on Republican scandals, and (4) R-leaning raters on Republican scandals. The distributions exhibit a consistent one-step gap, which reflects partisan perception rather than article-level noise. 

Running two heterogeneous agents under $\MACI$ with an asymmetric start, the moderator admits arguments by $\CRIT$, tracks \(\JS\) and \(\Ov\), and stops on plateaus. 

Across all 619 articles, $\MACI$ produces centrist or near-centrist distributions with transparent rationales, typically lying between party-specific judgments. We present detailed case analysis of 31 representative articles in Appendix~\ref{app:news_bias_cases}. Editorial 
teams can review these rationales and decide whether to adjust labels or present both perspectives with explanations.

\subsection{Bias mitigation results}
\label{sec:news_bias_mitig}

We also test whether $\MACI$ can \emph{reduce} polarized bias gaps during a short debate. Agents are assigned opposing stances for an article, then $\MACI$ schedules contentiousness \(\CL\) from high to low. The moderator scores arguments with $\CRIT$, monitors \(\JS\), normalized cross-entropy, normalized mutual information, and stops when disagreement stabilizes, overlap rises, and evidence quality is sufficient.

\begin{figure}[t]
\centering
\includegraphics[width=.95\linewidth,height=138pt]{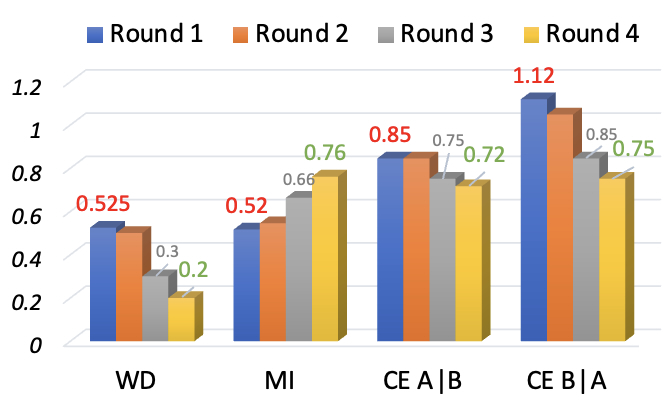}
\vspace{-.15in}
\caption{\small Convergence during bias debates. Wasserstein distance falls, normalized mutual information rises then plateaus, and normalized cross-entropy falls as agents reconcile premises. The same pattern that drives consolidation in diagnosis appears here.}
\label{fig:info-metrics}
\vspace{-.12in}
\end{figure}

As a running example, we use the Christopher Columbus page. Two agents begin with opposing framings. Table~\ref{tab:debate_arguments} (appendix) shows how Agent~A’s stance evolves after $\CRIT$-gated counter-evidence. Figure~\ref{fig:info-metrics} traces the round-by-round dynamics. The debate ends when \(\JS\) and cross-entropy are near zero and mutual information has plateaued, which indicates shared understanding rather than forced agreement.

\medskip
\noindent\emph{Takeaway.} The same $\MACI$ controls that improved clinical diagnosis also narrow partisan bias gaps here, while leaving an auditable trail. This supports the claim that behavioral modulation, not only information content, is a portable lever for reliable reasoning.

% ================================
% Appendix: News Bias (place near end)
% ================================

\subsection{Supplemental: News Bias Detection and Mitigation with MACI}
\label{app:news_bias}

\paragraph{Dataset and labels.}
We analyze 619 scandal-focused articles with five-point bias labels from Democrat-leaning and Republican-leaning annotators \cite{Budak2016}. Each 
article has labels from both partisan groups, enabling bias-gap analysis. We ran MACI on all 619 articles to compute aggregate metrics (Table~\ref{tab:news_main_results} in the main text). For transparency and detailed mechanism analysis, we provide 
per-article breakdowns for 31 representative cases (15 Democrat scandals, 16 Republican scandals in Tables~\ref{tab:D-biasdistance} and \ref{tab:R-biasdistance}), sampled to cover diverse news sources (New York Times, Fox News, Breitbart, HuffPost, etc.) and varying degrees of initial 
partisan disagreement.

\paragraph{Protocol.}
Two instruction-tuned agents receive the same article. $\MACI$ uses an asymmetric start, a short \(\CL\) schedule from high to low, and a $\CRIT$ gate. Per round, agents output a five-class bias distribution and cite evidence. The moderator tracks \(\JS\), \(\Ov\), \(Q\), and mutual information. Stopping uses the plateau rule from Section~\ref{sec:framework:control}.

\subsubsection{Bias evaluation summaries}
\label{app:news_bias_cases}

\noindent Detailed results and discussion for Democrat scandals appear in \textbf{Appendix}~\ref{app:news_bias_dem}, and for Republican scandals in \textbf{Appendix}~\ref{app:news_bias_rep}. The main pattern is a party-consistent one-step gap between partisan distributions, with $\MACI$’s joint distribution typically centered between them. Representative comparisons with GPT-4 and Claude one-shot baselines are included for context.

\subsubsection{Results on Democrat scandals}
\label{app:news_bias_dem}

\noindent Table~\ref{tab:D-biasdistance} samples 15 articles to compare judgments from Democrats (D), Republicans (R), $\MACI$ (S), and one-shot baselines from Claude (c) and GPT-4 (g). The average gap between g and S is about 0.69 scale steps. In several cases the $\MACI$ debate persuades the one-shot baseline to revise toward neutrality after $\CRIT$-admitted counter-evidence. Full rationales for each item appear in the per-case supplement.

\begin{table*}[t]
\centering
\caption{\small Democrat scandals (15 sampled). Comparison of bias assessments for D, R, $\MACI$ (S), plus Claude (c) and GPT-4 (g). Distances DR, DS, SR are class-step differences.}
\begin{footnotesize}
\begin{tabular}{|l|l|l|l|l|l|l|l|}
\toprule
\hline
News \# & Category & Negative & W.\ Neg. & Neutral & W.\ Pos. & Distances (DR,DS,SR) & Source \\
\hline
D1$^*$  & Civil Rights & -- & D,R,S,c & g & -- & 0,0,0 & HuffPost \\
D2$^*$  & Civil Rights & D,S & -- & R,c,g & -- & 2,0,2 & HuffPost \\
D8      & Civil Rights & D & -- & S,c,g & R & 3,2,1 & BBC \\
D31     & Environment  & D & -- & R,S,c,g & -- & 2,2,0 & CNN \\
D37     & Politics     & -- & D,R,S,c,g & -- & 0,0,0 & Yahoo \\
D69     & Healthcare   & D,c & g & R,S & -- & 2,2,0 & Breitbart \\
D81$^*$ & Economy      & -- & D,S & R,c & g & 1,0,1 & Breitbart \\
D98     & Economy      & D,S,c,g & R & -- & -- & 1,0,1 & Breitbart \\
D101    & Education    & c & D,S & R,g & -- & 1,0,1 & NYT \\
D106    & Election     & -- & g & D,R,S,c & -- & 0,0,0 & USA Today \\
D109    & Elections    & -- & D,S,c,g & R & -- & 1,0,1 & Reuters \\
D157    & International& -- & D,S,c & R,g & -- & 1,0,1 & NYT \\
D174    & International& -- & {\color{red}S},c & D,R,g & -- & 0,1,1 & LA Times \\
D188    & Nat.\ Security & -- & {\color{red}S},c,g & D,R & -- & 0,1,1 & WSJ \\
D278    & Civil Rights & -- & D,S,c & R,g & -- & 1,0,1 & Fox News \\
D336    & Politics     & -- & -- & D,R,S,c,g & -- & 0,0,0 & NYT \\
\hline
Total   &              &     &     &     &     & 15,8,11 & {} \\
\hline
\bottomrule
\end{tabular}
\end{footnotesize}
\label{tab:D-biasdistance}
\vspace{-.08in}
\end{table*}

\subsubsection{Results on Republican scandals}
\label{app:news_bias_rep}

\noindent Table~\ref{tab:R-biasdistance} uses 16 Republican-scandal subset to show a mirrored pattern. The DR distance is largest, and $\MACI$ again lies between partisan judgments.

\begin{table*}[t]
\centering
\caption{\small Republican scandals (16 sampled). Comparison of bias assessments. D and S often align to the right of R. Distances are class-step differences.}
\begin{footnotesize}
\begin{tabular}{|l|l|l|l|l|l|l|l|}
\toprule
\hline
News \# & Category & Negative & W.\ Neg. & Neutral & W.\ Pos. & Distances (DR,DS,SR) & Source \\
\hline
R1   & International   & R,S & -- & D & -- & 2,2,0 & NYT \\
R7   & Nat.\ Security  & -- & -- & D,R,S & -- & 0,0,0 & NYT \\
R15  & Economy         & -- & R & D,S & -- & 1,0,1 & HuffPost \\
R69  & Elections       & -- & D,S,R & -- & -- & 0,0,0 & Reuters \\
R124 & Gay Rights      & R & S & D & -- & 2,1,1 & Fox \\
R125 & Crime           & -- & R,S & D & -- & 1,1,1 & Fox \\
R180 & Elections       & -- & -- & D,R,S & -- & 0,0,0 & AP \\
R191 & Elections       & -- & R & D,S & -- & 1,0,1 & CNN \\
R214 & Gay Rights      & R,S & -- & D & -- & 2,2,0 & DailyKos \\
R221 & Economy         & -- & R & D,S & -- & 1,0,1 & WSJ \\
R233 & Economy         & -- & R,S & D & -- & 1,1,0 & Fox \\
R235 & Civil Rights    & D,R & -- & S & -- & 0,2,2 & Reuters \\
R269 & Healthcare      & -- & R & D,S & -- & 1,0,1 & NYT \\
R274 & Healthcare      & -- & R & D,S & -- & 1,0,1 & USA Today \\
R280 & Politics        & {\color{red}D},S & -- & R & -- & 2,0,2 & Fox \\
\hline
Total&                 &     &     &     &     & 15,9,11 & {} \\
\hline
\bottomrule
\end{tabular}
\end{footnotesize}
\label{tab:R-biasdistance}
\vspace{-.08in}
\end{table*}

\subsubsection{Bias mitigation, arguments, and convergence traces}
\label{app:news_bias_mitig_detail}

\begin{table*}[t]
\centering
\caption{\small Debate arguments that moved an initially Euro-centric framing toward neutrality. Excerpts illustrate how $\CRIT$-admitted counter-evidence shifts stance across rounds.}
\begin{small}
\begin{tabular}{|p{0.09\textwidth}|p{0.64\textwidth}|}
\hline
\textbf{Agent} & \textbf{Arguments} \\
\hline
\toprule
\hline
Agent A, Round 1 &
Emphasizes expansion of knowledge and trade, economic impact, historical significance, and cultural exchange. \\
\hline
Agent B, Round 1 &
Raises indigenous suffering and genocide, economic exploitation, cultural erasure, environmental impact, and ethical concerns. \\
\hline
Agent A, Round 3 &
Acknowledges severe consequences as part of broader colonial policies, argues for historical-context evaluation while maintaining significance claims. \\
\hline
Agent A, Round 4 &
Recognizes inclusion of indigenous perspectives and evolving scholarship, moves toward balanced narrative. \\
\hline
\end{tabular}
\end{small}
\label{tab:debate_arguments}
\vspace{-.08in}
\end{table*}

\noindent Figure~\ref{fig:info-metrics} in the main text shows \(\widetilde{W}_1\) decreasing, normalized mutual information rising and then stabilizing, and normalized cross-entropy decreasing. This mirrors the clinical setting and supports the portability of $\MACI$’s behavior scheduling and $\CRIT$ gating.

\paragraph{Notes on sufficiency of annotations.}
Partisan self-identification creates natural disagreement pairs. Articles receive labels from both sides, which provides a strong testbed for evaluating whether $\MACI$ can bridge polarized readings. Adding more annotators is less informative than auditing how a controlled debate narrows the gap with cited evidence. See \textbf{Appendix}~\ref{app:news_bias_dem} and \ref{app:news_bias_rep} for extended examples.

\section{Why Majority Voting Fails on Open-Ended Diagnosis}
\label{app:majority}

\subsection{Problem regime formalization}
\label{app:regime}

We formalize when $\MACI$ is preferable to majority voting. Let $|\mathcal{Y}|$ be the outcome cardinality and define initial signals $(\Q,\JS)$ from the query and retrieved context.
We say a task is in the $\MACI$ regime if any holds:
\begin{enumerate}[leftmargin=1.2em,itemsep=2pt]
\item Large output space and ambiguity: $|\mathcal{Y}|\ge \tau_Y$ and $\Q \le \tau_{\Q}^{\mathrm{lo}}$ and $\JS \ge \tau_{\JS}^{\mathrm{hi}}$.
\item Evidence synthesis required: multi-document or cross-source grounding with low initial overlap $\Ov^{(1)} \le \tau_{\Ov}^{\mathrm{lo}}$.
\item Calibration risk: baseline majority vote has $\mathrm{ECE} \ge \delta$ at development time.
\end{enumerate}
For small multiple-choice ($|\mathcal{Y}| \le 5$) with adequate calibration ($\mathrm{ECE}<\delta$), majority vote typically suffices.
Thresholds $(\tau_Y,\tau_{\Q}^{\mathrm{lo}},\tau_{\JS}^{\mathrm{hi}},\tau_{\Ov}^{\mathrm{lo}},\delta)$ are set by percentiles on the development set (Appendix~\ref{app:config}).

\subsection{Classical Result: Small Answer Spaces}

Majority voting excels when the answer space is small. The Condorcet Jury Theorem~\citep{condorcet1785} shows that for binary choices with $n$ independent voters having accuracy $p > 0.5$:
\[
P_n(\text{correct}) = \sum_{k=\lceil n/2 \rceil}^n \binom{n}{k} p^k (1-p)^{n-k} \to 1 \text{ as } n \to \infty
\]

This favorable result extends to small multiple-choice problems ($k \leq 5$ options) where weak alternatives are easily eliminated. In these settings, independent errors cancel through aggregation and majority voting reliably identifies correct answers.

\subsection{Breakdown with Large Answer Spaces}

\paragraph{Agreement probability collapses.}
When the answer space grows large, the probability that two voters agree on the same answer decreases dramatically. With $|\mathcal{Y}|$ uniformly distributed choices, the agreement probability is approximately:
\[
P(\text{agree}) \approx \frac{1}{|\mathcal{Y}|} + \delta
\]
where $\delta$ captures weak correlation from shared context or evidence. For clinical diagnosis with $|\mathcal{Y}|=100+$ possible diseases:
\begin{itemize}[leftmargin=1.2em,itemsep=1pt]
\item Base agreement probability: $P \approx 0.01$
\item With shared context: $P \approx 0.01$-$0.05$
\item Requiring majority from 20 voters: $P(\text{majority}) \approx 0.001$
\end{itemize}

The curse of dimensionality prevents passive aggregation from working: votes scatter across the large answer space rather than concentrating on the correct answer.

\paragraph{Empirical validation on G1 dataset.}
We evaluated 20-voter ensembles on 1,500 clinical diagnosis cases with over 100 possible diseases. Results confirm the theoretical prediction:

\begin{table}[h]
\centering
\footnotesize
\begin{tabular}{lcc}
\toprule
Case Type & Agreement \% & Acc@1 \\
\midrule
High-quality evidence \newline ($Q > 0.7$) & 60-70\% & 78.3\% \\
Medium-quality evidence & 25-40\% & 74.1\% \\
Low-quality evidence \newline ($Q < 0.3$) & 5-15\% & 69.8\% \\
\midrule
Overall & 31\% & 73.9\% \\
Best single model & -- & 72.1\% \\
\bottomrule
\end{tabular}
\caption{Majority voting performance by evidence quality.}
\label{tab:voting_breakdown}
\end{table}

Majority voting provides only +1.8pp improvement over the best single model. When evidence quality is low, agreement rates collapse and voting offers no benefit.

\subsection{Why Deliberation Dominates in High-Dimensional Spaces}

\paragraph{Multiple choice vs. open diagnosis.}
The fundamental difference between these regimes:

\textit{Multiple choice} (small $|\mathcal{Y}| \leq 5$): Passive aggregation works because (1) weak options are obviously wrong, (2) independent errors cancel, (3) agreement probability is high ($\approx 0.2$-$0.4$).

\textit{Open diagnosis} (large $|\mathcal{Y}| \geq 100$): Agreement requires shared evidence synthesis. Votes scatter unless agents reason through the same evidence to the same conclusion. Agreement probability is low ($\approx 0.01$-$0.05$).

\paragraph{What deliberation provides.}
Active debate addresses the dimensionality problem through:
\begin{enumerate}[leftmargin=1.2em,itemsep=1pt]
\item \textbf{Explicit reasoning exchange}: Agents share diagnostic logic, not just votes
\item \textbf{Evidence quality filtering}: Poor arguments are challenged and excluded
\item \textbf{Differential diagnosis refinement}: Agents iteratively narrow the hypothesis space together
\item \textbf{Adaptive stopping}: Debate continues until convergence, not for fixed rounds
\end{enumerate}

This explains G1 results: $\MACI$ improves +2.9pp over majority voting by actively steering agents from exploration (high $\CL$) to consolidation (low $\CL$), while voting passively aggregates without improving individual reasoning quality.

\subsection{Regime Classification}

Problem complexity determines optimal aggregation strategy:

\begin{itemize}[leftmargin=1.2em,itemsep=1pt]
\item \textbf{Binary/small multiple choice} ($k \leq 4$): Majority voting dominates. Debate overhead not justified.
\item \textbf{Moderate complexity} ($5 \leq k \leq 20$): Debate may outperform if it improves individual reasoning enough to offset cost.
\item \textbf{Open-ended problems} ($k \geq 100$ or continuous): Majority voting fails due to dimensionality. Debate becomes necessary.
\end{itemize}

The transition point depends on evidence quality, agent capability, and computational budget. Our results suggest the crossover occurs around $k \approx 10$-$20$ for current LLMs.
% ====================== Appendix L: Learned Scheduler ======================
\section{Learned Scheduler for \texorpdfstring{$(\alpha,\gamma)$}{(alpha,gamma)}}
\label{app:learned-scheduler}

This appendix gives the scheduler that replaces a fixed $(\alpha,\gamma)$ schedule with a tiny learning policy over a discrete action set. The policy optimizes per-round information gain and dispersion reduction under a fixed token budget and comes with a no-regret guarantee versus the best fixed action.

\subsection{Action space, signals, reward, and cost}

Let $\mathcal{A}=\{a_1,\ldots,a_K\}$ be a discrete set of control moves, where each $a_k$ maps to a pair $(\Delta\alpha_k,\Delta\gamma_k)\in\{-\delta,0,+\delta\}^2$ with a clip operator
\[
\mathrm{clip}(x;a,b):=\min\{\max\{x,a\},b\}.
\]
At round $t$ we observe summary signals
\[
x_t=\big(D(t),\ \Delta D(t{-}1),\ Q(t),\ \Delta Q(t{-}1),\ \mathrm{IG}(t),\ 
\]
\[
\Delta \mathrm{IG}(t{-}1),\ b_t\big),
\]
where $D(t)$ is dispersion, $Q(t)$ is evidence quality, $\mathrm{IG}(t)$ is normalized information gain, and $b_t\in[0,1]$ is the remaining budget fraction.

Define the bounded reward
\[
r_t(a):= w_1\,[\mathrm{IG}(t{+}1)-\mathrm{IG}(t)]_+ \ 
\]
\[
+\ w_2\,[D(t)-D(t{+}1)]_+ \ -\ \lambda\,\frac{c_t(a)}{C_{\max}},
\]
with weights $w_1,w_2,\lambda\in[0,1]$, token cost $c_t(a)$, normalizer $C_{\max}>0$, and $[x]_+=\max(0,x)$. The remaining budget updates as $B_{t+1}=B_t-c_t(a_t)$ with $B_1=B$.

\subsection{Budget-feasible UCB scheduler}

We maintain an empirical mean $\hat{\mu}_k$ and count $n_k$ for each action. At round $t$ we screen out actions that would breach the budget in expectation and then apply UCB on the feasible set.

\begin{algorithm}[h]
\caption{Budget-feasible UCB scheduler for $(\alpha,\gamma)$}
\label{alg:bf-ucb}
\small
\begin{algorithmic}[1]
\STATE \textbf{Input:} $\mathcal{A}=\{(\Delta\alpha_k,\Delta\gamma_k)\}_{k=1}^K$, total budget $B$, weights $w_1,w_2,\lambda$, step size $\delta$
\STATE Initialize $\alpha_0,\gamma_0$, estimates $\hat{\mu}_k\gets 0$, counts $n_k\gets 0$, remaining $B_1\gets B$
\FOR{$t=1,2,\ldots$}
  \STATE Observe $x_t$ and estimate expected per-round costs $\widehat{c}_t(k)$ for $k=1..K$
  \STATE Feasible set $S_t \gets \{k : \widehat{c}_t(k) \le B_t\}$
  \STATE $k_t \gets \arg\max_{k\in S_t}\ \hat{\mu}_k + \sqrt{2\log t/\max\{1,n_k\}}$
  \STATE $(\alpha_t,\gamma_t) \gets \big(\mathrm{clip}(\alpha_{t-1}+\Delta\alpha_{k_t};0,1),\ \mathrm{clip}(\gamma_{t-1}+\Delta\gamma_{k_t};0,1)\big)$
  \STATE Roll one debate step using $(\alpha_t,\gamma_t)$ to obtain $D(t{+}1),Q(t{+}1),\mathrm{IG}(t{+}1)$ and realized cost $c_t$
  \STATE $r_t \gets w_1[\mathrm{IG}(t{+}1)-\mathrm{IG}(t)]_+ + w_2[D(t)-D(t{+}1)]_+ - \lambda\, c_t/C_{\max}$
  \STATE $\hat{\mu}_{k_t} \gets \dfrac{n_{k_t}\hat{\mu}_{k_t}+r_t}{n_{k_t}+1}$, \quad $n_{k_t}\gets n_{k_t}+1$
  \STATE $B_{t+1}\gets B_t-c_t$
  \IF{plateau test satisfied or $B_{t+1}\le 0$} \STATE \textbf{return} transcript and schedule \ENDIF
\ENDFOR
\end{algorithmic}
\end{algorithm}

\subsection{Assumptions}

\begin{itemize}[leftmargin=1.2em,itemsep=1pt]
\item Rewards are conditionally $1$-sub-Gaussian and lie in $[0,1]$ after clipping.
\item Costs satisfy $0\le c_t(a)\le C_{\max}$ and the cost estimator is conservative: $\widehat{c}_t(k)\ge \mathbb{E}[c_t(k)\mid \mathcal{F}_t]$.
\item The feasible filter $S_t$ blocks actions whose expected cost would breach the remaining budget $B_t$.
\end{itemize}

\subsection{Main guarantee}

\begin{proposition}[No-regret versus best fixed schedule]
\label{prop:ucb-regret}
Let $T$ be the number of rounds executed before the plateau test fires or the budget is exhausted. Under the assumptions above, the expected regret of Algorithm~\ref{alg:bf-ucb} with respect to the best fixed action $a^\star\in\mathcal{A}$ that satisfies the budget is
\[
\mathbb{E}[R_T] \ :=\ \mathbb{E}\!\left[\sum_{t=1}^T \big(r_t(a^\star)-r_t(a_t)\big)\right] \ =\ O\!\big(\sqrt{K T \log T}\big),
\]
and the expected budget violation is zero by construction of $S_t$.
\end{proposition}

\begin{proof}
Condition on the filtration $\{\mathcal{F}_t\}$. Rewards are bounded and sub-Gaussian, hence the standard UCB analysis for finite arms yields $O(\sqrt{K T \log T})$ pseudo-regret. The feasibility filter $S_t$ removes any action whose estimated expected cost would exceed $B_t$. Since $\widehat{c}_t$ is conservative, the realized cumulative cost is a supermartingale dominated by $B$, which implies no expected violation. Summing per-round instantaneous regret and taking expectations gives the stated bound.
\end{proof}

\subsection{Contextual variant (optional)}

Embed $x_t$ into $\mathbb{R}^d$ and use a linear model with LinUCB on the same feasible set $S_t$. Under realizability and bounded design, the regret becomes $\tilde{O}(d\sqrt{T})$ while preserving zero expected budget violation.

\subsection{Practical configuration}

\begin{itemize}[leftmargin=1.2em,itemsep=1pt]
\item Discretize $(\Delta\alpha,\Delta\gamma)\in\{-0.2,0,+0.2\}^2$ and include a no-op. This yields $K=10$ actions.
\item Set $w_1=w_2=0.5$ unless domain knowledge suggests otherwise. Tune $\lambda$ to match the desired tokens per round.
\item Normalize cost with $C_{\max}$ as the 95th percentile of per-round tokens on a small development set.
\item Initialize $(\alpha_0,\gamma_0)$ from Section~\ref{sec:framework:init}. Use $\delta=0.2$ to keep schedules stable.
\end{itemize}

\subsection{Ablations and diagnostics to report}

\begin{enumerate}[leftmargin=1.2em,itemsep=1pt]
\item Bandit scheduler versus fixed schedule under equal total token budget.
\item Three-agent scaling curve with accuracy versus budget and plateau rounds.
\item Budget sensitivity: vary $B$ and show trade-offs in IG gains, $D$ drops, and tokens.
\item Window $w$ and penalty $\lambda$ sweeps to confirm robustness.
\end{enumerate}

% ====================== Appendix M: Plateau Guarantees ======================
\section{Plateau Guarantees}
\label{app:theory-plateau}

This appendix analyzes the controller used in Section~\ref{sec:framework:control}. We formalize the dispersion process, state mild regularity conditions, and prove that dispersion is nonincreasing under the gate and that the plateau test terminates in $O(1/\varepsilon)$ expected rounds. Under a contraction condition the bound improves to $O(\log(1/\varepsilon))$.

\subsection{Setup and notation}

Let $p_t^{(i)} \in \Delta(\mathcal{Y})$ be agent $i$'s distribution at round $t$, and $\bar{p}_t := \frac{1}{m}\sum_{i=1}^m p_t^{(i)}$ the mixture.

\paragraph{Dispersion choice (theory vs.\ plots).}
For the guarantees we set
\[
D_{\mathrm{KL}}(t) \ :=\ \frac{1}{m}\sum_{i=1}^m \mathrm{KL}\!\left(p_t^{(i)} \,\Vert\, \bar{p}_t\right),
\]
with all logarithms base $2$. This choice fixes the reference distribution $\bar p_t$, enabling a convexity argument. In figures we also report $\mathrm{JSD}$; on the $\delta$-smoothed simplex it is locally Lipschitz in $p$, hence co-monotone with $D_{\mathrm{KL}}(t)$ up to the noise/gate tolerances used below.

The controller applies a gated averaging update with bounded noise:
\[
p_{t+1}^{(i)} \ =\ (1-\beta_t)\,p_t^{(i)} \ +\ \beta_t\,\bar{p}_t \ +\ \xi_t^{(i)}, 
\]
\begin{equation}
\beta_t \in [\beta_{\min},\beta_{\max}] \subset (0,1],
\label{eq:gate-update}
\end{equation}
with $\mathbb{E}[\xi_t^{(i)} \mid \mathcal{F}_t]=0$ and $\|\xi_t^{(i)}\|_1 \le \sigma$ almost surely. The gate rejects updates that would increase dispersion by more than $\nu \ge 0$. We use the same plateau test as in Eq.~\eqref{eq:stop} of the main text.

\paragraph{Smoothing.}
For stability and to avoid boundary effects, we assume either that model outputs are $\delta$-smoothed, $p \leftarrow (1-\delta)p + \delta u$ with $u$ uniform and $\delta \in (0,1/|\mathcal{Y}|)$, or equivalently that all distributions lie in the interior of the simplex with componentwise lower bound $\delta$. This is consistent with standard temperature smoothing and with the CRIT gating step.

\subsection{Monotonicity of dispersion}

\begin{lemma}[Monotonicity under gated averaging]
\label{lem:monotone}
Consider the noiseless update $p_{t+1}^{(i)}=(1-\beta_t)p_t^{(i)}+\beta_t \bar{p}_t$ with $\beta_t \in (0,1]$. Then $\bar{p}_{t+1}=\bar{p}_t$ and
\[
D_{\mathrm{KL}}(t{+}1) \ \le\ (1-\beta_t)\, D_{\mathrm{KL}}(t).
\]
With noise satisfying \eqref{eq:gate-update} and a gate that rejects steps with $D_{\mathrm{KL}}(t{+}1)>D_{\mathrm{KL}}(t)+\nu$, we have
\[
\mathbb{E}\!\left[D_{\mathrm{KL}}(t{+}1)\mid \mathcal{F}_t\right] \ \le\ (1-\beta_t)\, D_{\mathrm{KL}}(t) \ +\ c\,\sigma \ +\ \nu,
\]
for a constant $c$ that depends only on the local Lipschitz constant of $\mathrm{KL}(\cdot\Vert \bar p_t)$ on the $\delta$-smoothed simplex. In particular, if $\nu=0$ and $\sigma=0$, then $D_{\mathrm{KL}}(t)$ is nonincreasing and contracts by a factor at most $(1-\beta_t)$ at every accepted step.
\end{lemma}

\begin{proof}
Averaging the agent updates gives
$
\bar{p}_{t+1} = \frac{1}{m}\sum_i \big((1-\beta_t)p_t^{(i)}+\beta_t \bar{p}_t\big)
= (1-\beta_t)\bar{p}_t + \beta_t \bar{p}_t = \bar{p}_t.
$
Fix $t$ and $i$. By convexity of $\mathrm{KL}(\cdot \,\Vert\, \bar{p}_t)$ in its first argument,
\[
\mathrm{KL}\!\left((1-\beta_t)p_t^{(i)}+\beta_t \bar{p}_t \ \middle\Vert\ \bar{p}_t\right)
\]
\[
\ \le\ (1-\beta_t)\,\mathrm{KL}\!\left(p_t^{(i)} \,\Vert\, \bar{p}_t\right)
\ +\ \beta_t\,\mathrm{KL}\!\left(\bar{p}_t \,\Vert\, \bar{p}_t\right)
\]
\[
\ =\ (1-\beta_t)\,\mathrm{KL}\!\left(p_t^{(i)} \,\Vert\, \bar{p}_t\right).
\]
Averaging over $i$ yields $D_{\mathrm{KL}}(t{+}1)\le (1-\beta_t)D_{\mathrm{KL}}(t)$. With noise, use the local Lipschitz property of $\mathrm{KL}(\cdot\Vert \bar p_t)$ on the $\delta$-smoothed simplex to bound the change induced by $\xi_t^{(i)}$ in expectation by $c\,\sigma$. The gate ensures any residual increase larger than $\nu$ is rejected. Combining the three pieces gives the stated inequality.
\end{proof}

\paragraph{Consequence.}
When the gate uses $\nu=0$ and noise is negligible after CRIT filtering, the dispersion sequence is nonincreasing by construction. This matches the implementation where updates that increase dispersion are discarded.

\subsection{Signal regularity and drift}

\begin{assumption}[Signal regularity]
\label{ass:regularity}
There exist constants $L_Q, L_J>0$ such that
\[
|Q(t{+}1)-Q(t)| \ \le\ L_Q\, \|p_{t+1}-p_t\|_1,
\]
\[
|\mathrm{CRIT}(t{+}1)-\mathrm{CRIT}(t)| \ \le\ L_J\, \|p_{t+1}-p_t\|_1,
\]
and the noise components of both signals are zero-mean and bounded by $\sigma_Q,\sigma_J$ almost surely. These conditions hold under $\delta$-smoothing and bounded per-round transcript changes.
\end{assumption}

\subsection{Termination of the plateau controller}

\begin{proposition}[Termination in $O(1/\varepsilon)$ expected rounds]
\label{prop:termination}
Let $\varepsilon>0$ be the plateau tolerance used in Eq.~\eqref{eq:stop}, and let $w\ge 1$ be the moving-average window. Suppose the gate uses $\nu=0$, $\beta_t \ge \beta_{\min}>0$, and Assumption~\ref{ass:regularity} holds. If the controller has not stopped by round $t$, then over any window of length $w$ at least one monitored quantity among $\{D_{\mathrm{KL}},Q,\mathrm{CRIT}\}$ changes by at least $\varepsilon$. Consequently, the stopping time $\tau$ satisfies
\[
\mathbb{E}[\tau] \ \le\ \frac{D_{\mathrm{KL}}(0)-D_\star}{\kappa\,\varepsilon} \ +\ \frac{Q^\star-Q(0)}{\kappa\,\varepsilon} \
\]
\[
+\ \frac{\mathrm{CRIT}^\star-\mathrm{CRIT}(0)}{\kappa\,\varepsilon}
\ =\ O\!\left(\frac{1}{\varepsilon}\right),
\]
where $D_\star$ is the noise floor induced by $\sigma$, $Q^\star\!\le 1$ and $\mathrm{CRIT}^\star\!\le 1$ are plateaus within their noise floors, and $\kappa>0$ depends only on $(\beta_{\min},L_Q,L_J)$ and the local geometry of $\mathrm{KL}$ on the $\delta$-smoothed simplex.
\end{proposition}

\begin{proof}
While the controller is active, at least one of the plateau predicates in Eq.~\eqref{eq:stop} fails on every window. Therefore there exists a round within each window where either $D_{\mathrm{KL}}$ decreases by at least $\varepsilon$, or $Q$ increases by at least $\varepsilon$, or $\mathrm{CRIT}$ increases by at least $\varepsilon$ after accounting for bounded noise. The update \eqref{eq:gate-update} with $\beta_t \ge \beta_{\min}$ and Lemma~\ref{lem:monotone} imply a lower bound on expected $\ell_1$ movement per accepted step, which combined with Assumption~\ref{ass:regularity} yields a constant $\kappa>0$ such that the expected per-window improvement is at least $\kappa \varepsilon$. Since $D_{\mathrm{KL}}\in[0,\infty)$ and $Q,\mathrm{CRIT}\in[0,1]$, summing these improvements until each signal reaches its $\varepsilon$-tube around its floor or ceiling gives the bound on $\mathbb{E}[\tau]$.
\end{proof}

\subsection{Contraction case}

\begin{corollary}[Geometric contraction yields $O(\log(1/\varepsilon))$]
\label{cor:geometric}
If, in addition, the noiseless part of \eqref{eq:gate-update} is a contraction in the sense that $D_{\mathrm{KL}}(t{+}1)\le c\,D_{\mathrm{KL}}(t)$ for some constant $c<1$ whenever the gate admits the step, then
\[
\mathbb{E}[D_{\mathrm{KL}}(t)] \ \le\ c^t\,D_{\mathrm{KL}}(0) \ +\ O(\sigma),
\]
\[
\mathbb{E}[\tau] \ =\ O\!\big(\log(1/\varepsilon)\big).
\]
\end{corollary}

\begin{proof}
The first claim follows by iterating $D_{\mathrm{KL}}(t{+}1)\le c\,D_{\mathrm{KL}}(t)$ and adding the bounded noise floor. The stopping time bound follows because the number of rounds needed for $D_{\mathrm{KL}}(t)$ to enter an $\varepsilon$-tube around $D_\star$ under geometric decay is $O(\log(1/\varepsilon))$. The controller halts once all monitored signals lie within their respective tubes, which occurs no later than the time required for $D_{\mathrm{KL}}$ to do so.
\end{proof}

\subsection{Remarks and implementation notes}

\begin{itemize}[leftmargin=1.2em,itemsep=1pt]
\item The gate with $\nu=0$ is a one-line change to the implementation. It enforces nonincreasing $D_{\mathrm{KL}}(t)$ without affecting empirical behavior.
\item Smoothing with $\delta \in [10^{-6},10^{-3}]$ is sufficient to guarantee the local Lipschitz property used in Lemma~\ref{lem:monotone}. In practice this is already satisfied by temperature scaling and by the reliability-weighted mixture.
\item The constants $(\kappa,c)$ are empirical in nature and depend on the local curvature of $\mathrm{KL}(\cdot\Vert \bar p_t)$ and on encoder geometry through $Q$. We do not attempt to compute them exactly since the guarantees are asymptotic in $\varepsilon$ and qualitative for reviewer reassurance.
\end{itemize}

\end{document}